\RequirePackage{fix-cm}
\documentclass[twocolumn,natbib]{svjour3}          
\usepackage{hyperref}

\smartqed  
\usepackage{graphicx}
\usepackage{amsmath,amssymb} 
\usepackage[numbered]{algo}
\usepackage{color}
\usepackage{framed}

\usepackage{url}

\usepackage{subfigure}

\definecolor{blue}{rgb}{0,0,0.9}
\definecolor{red}{rgb}{0.9,0,0}
\definecolor{green}{rgb}{0,0.9,0}

\newcommand{\argmin}{\mathop{\mathrm{argmin}}}
\def\diag#1{\mbox{diag}(#1)}

\usepackage{algorithm}
\usepackage{algorithmic}

\usepackage{appendix}
 \journalname{International Journal of Computer Vision}
\def\R{\mathbb{R}}
\def\norm#1{\|#1\|}
\def\inprod#1#2{\langle #1, #2\rangle}
\def\dN{\delta N}
\def\DD{\mathbb{D}} \def\QQ{\mathbb{Q}}
\def\AA{\mathbb{A}}
\def\cP{\mathcal{P}}

\def\cX{{\cal X}}
\def\cY{{\cal Y}}
\def\hx{\widehat{x}} \def\hy{\widehat{y}}
\def\tx{\widetilde{x}}

\begin{document}

\title{Practical Matrix Completion and Corruption Recovery using Proximal Alternating Robust Subspace Minimization
}

\titlerunning{Practical Matrix Completion and corruption recovery using PARSuMi}        

\author{Yu-Xiang Wang         \and
        Choon Meng Lee \and
        Loong-Fah Cheong \and
        Kim-Chuan Toh  
}

\authorrunning{Wang, Lee, Cheong and Toh} 

\institute{Y.X. Wang, C.M. Lee, L.F. Cheong, K.C. Toh \at
              National University of Singapore \\
              \email{\{yuxiangwang, leechoonmeng, eleclf, mattohkc\} @nus.edu.sg}             \\
The support of the PSF grant 1321202075 is gratefully acknowledged.
}

\date{Received: 5 September 2013 / Accepted: 26 June 2014}

\maketitle

\begin{abstract}
   Low-rank matrix completion is a problem of immense practical importance. Recent works on the subject often use nuclear norm as a convex surrogate of the rank function. Despite its solid theoretical foundation, the convex version of the problem often fails to work satisfactorily in real-life applications. Real data often suffer from very few observations, with support not meeting the randomness requirements, ubiquitous presence of noise and potentially gross corruptions, sometimes with these simultaneously occurring.

This paper proposes a Proximal Alternating Robust Subspace Minimization (PARSuMi) method to tackle the three problems. The proximal alternating scheme explicitly exploits the rank constraint on the completed matrix and uses the $\ell_0$ pseudo-norm directly in the corruption recovery step.
We show that the proposed method for the non-convex and non-smooth model converges to a stationary point. Although it is not guaranteed to find the global optimal solution, in practice
we find that our algorithm can typically arrive at a good local minimizer when it is
supplied with a reasonably good starting point based on convex optimization.
Extensive experiments with challenging synthetic and real data demonstrate that our algorithm succeeds in a much larger range of practical problems where convex optimization fails, and
it also outperforms various state-of-the-art algorithms.


\end{abstract}

\section{Introduction}
\label{intro}
Completing a low-rank matrix from partially observed entries, also known as matrix completion, is a central task in many real-life applications. The same abstraction of this problem has appeared in diverse fields such as signal processing, communications, information retrieval, machine learning and computer vision. For instance, the missing data to be filled in may correspond to plausible movie recommendations \citep{koren2009MF,funk2006netflix}, occluded feature trajectories for rigid or non-rigid structure from motion, namely SfM \citep{hartley2003powerfactorization,Damped_Newton_2005} and NRSfM \citep{Paladini_NRSfM}, relative distances of wireless sensors \citep{oh2010sensor}, pieces of uncollected measurements in DNA micro-array \citep{DNA06}, just to name a few.
\begin{figure}[ht]
\centering
 \includegraphics[width=0.9\linewidth, height=0.5\linewidth]{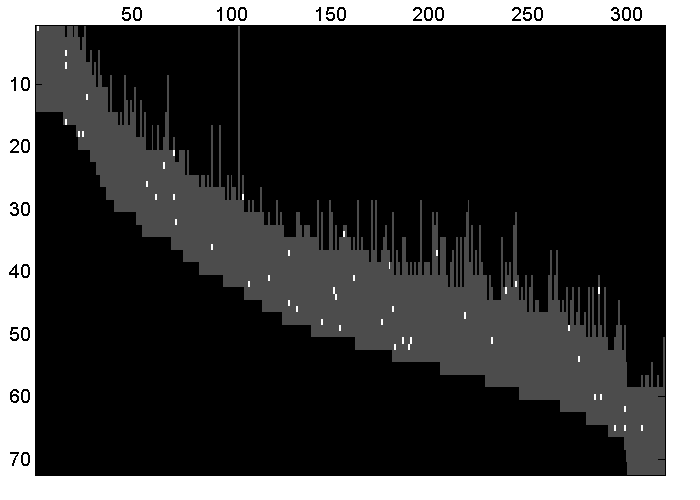}
\caption{
Sampling pattern of the Dinosaur sequence: 316 features are tracked over 36 frames. Dark area represents locations where no data is available; sparse highlights are injected gross corruptions. Middle stripe in grey are noisy observed data, occupying 23\% of the full matrix. The task of this paper is to fill in the missing data and recover the corruptions.
}
\label{fig:DinosaurSampling}
\end{figure}

The common difficulty of these applications lies in the scarcity of the observed data, uneven distribution of the support, noise, and more often than not, the presence of gross corruptions in some observed entries. For instance, in the movie rating database Netflix \citep{netflix}, only less than 1\% of the entries are observed and 90\% of the observed entries correspond to 10\% of the most popular movies. In photometric stereo, the missing data and corruptions (arising from shadow and specular highlight as modeled in \citet{Wu_photometric}) form contiguous blocks in images and are by no means random. In structure from motion, the observations fall into a diagonal band shape, and feature coordinates are often contaminated by tracking errors (see the illustration in Fig.~\ref{fig:DinosaurSampling}). Therefore, in order for any matrix completion algorithm to work in practice, these aforementioned difficulties need to be tackled altogether. We refer to this problem as \textbf{practical matrix completion}.
Mathematically, the problem to be solved is the following:
\begin{equation*}
\boxed{
\begin{aligned}
& \underset{}{\text{Given}}
& & \Omega, \widehat{W}_{ij}\text{ for all }(i,j)\in\Omega,  \\
& \underset{}{\text{find}}
& & W, \tilde{\Omega},\\
& \text{s.t.}
& & \mathrm{rank}(W)\text{ is small};\; \mathrm{card}(\tilde{\Omega})\text{ is small};\\
& & & |W_{ij}-\widehat{W}_{ij}| \text{ is small } \forall(i,j)\in\Omega|\tilde{\Omega}.
\end{aligned}
}
\end{equation*}
where $\Omega$ is the index set of observed entries whose locations are not necessarily
selected at random, $\tilde{\Omega}\in\Omega$ represents the index set of corrupted data, $\widehat{W}\in \mathbb{R}^{m\times n}$ is the measurement matrix with only $\widehat{W}_{ij\in{\Omega}}$ known, i.e., its support is
contained in $\Omega$. Furthermore, we define the projection $\mathcal{P}_{\Omega}:\mathbb{R}^{m\times n}\mapsto \mathbb{R}^{|\Omega|}$ so that $\mathcal{P}_{\Omega}(\widehat{W})$ denotes the vector of observed data. The adjoint of $\mathcal{P}_\Omega$ is
denoted by  $\mathcal{P}_\Omega^*$.

Extensive theories and algorithms have been developed to tackle some aspect of the
challenges listed in the preceding paragraph,
but those tackling the full set of challenges are far and few between, thus resulting in a dearth of practical algorithms. Two dominant classes of approaches are nuclear norm minimization, e.g. \citet{Candes_ExactMC,Candes2011_JACM,CandesNoise,chen2011erasures},
and matrix factorization, e.g., \citet{koren2009MF,Damped_Newton_2005,WibergL2,Subspace_ChenPei_2008,WibergL1}. Nuclear norm minimization methods minimize the convex relaxation of rank
instead of the rank itself,
and are supported by rigorous theoretical analysis and efficient
numerical computation. However, the conditions under which they succeed are often too restrictive for it to work well in real-life applications (as reported in \citet{shi2011limitations} and \citet{jain2012alt_min_global}).
In contrast,  matrix factorization is widely used in practice and are considered very effective for problems such as movie recommendation \citep{koren2009MF} and structure from motion \citep{Kanade_Factorization,Paladini_NRSfM} despite its lack of rigorous theoretical foundation. Indeed, as one factorizes matrix $W$ into $UV^T$, the formulation becomes bilinear and thus optimal solution is hard to obtain except in very specific cases (e.g., in \citet{jain2012alt_min_global}). A more comprehensive survey of the algorithms and review of the strengths and weaknesses will be given in the next section.

In this paper, we attempt to solve the practical matrix completion problem under the prevalent case where the rank of the matrix $W$ and the cardinality of $\tilde{\Omega}$ are upper bounded by some known parameters $r$ and $N_0$ via the
following non-convex, non-smooth optimization model:
\begin{eqnarray}\label{eq:MC_Formulation_0}
\begin{array}{rl}
 \underset{W,E}{\text{min}}\;
 & \frac{1}{2}\|\mathcal{P}_{\Omega}(W-\widehat{W}+E)\|^2 + \frac{\epsilon}{2}
\norm{\mathcal{P}_{\overline{\Omega}}(W)}^2
\\[5pt]
 \text{s.t.} \;
 & \mathrm{rank}(W)\leq r, \; W\in \R^{m\times n}
\\[5pt]
 & \|E\|_0\leq N_0, \;  \norm{E} \leq K_E, \; E\in \R^{m\times n}_\Omega
\end{array}
\end{eqnarray}
where
$\R^{m\times n}_\Omega$ denotes the set of $m\times n$ matrices  whose
supports are subsets of $\Omega$ and $\norm{\cdot}$ is the Frobenius norm; $K_E$ is a finite constant introduced to facilitate the convergence proof.
Note that the restriction of $E$ to $\R^{m\times n}_\Omega$ is natural since the role of $E$ is to capture the gross corruptions in the observed data $\widehat{W}_{ij\in\Omega}$. The bound constraint on $\norm{E}$ is natural in some problems when the true matrix $W$ is bounded (e.g., Given the typical movie ratings of 0-10, the gross outliers can only lie in [-10, 10]). In other problems, we simply choose $K_E$ to be some large multiple (say 20) of $\sqrt{N_0}\times {\rm median}(\mathcal{P}_\Omega(\widehat{W}))$, so that the constraint is essentially inactive and has no impact on the optimization.
Note that without making any randomness assumption on the index set $\Omega$ or
assuming that the problem has a unique solution $(W^*,E^*)$ such that the singular
vector matrices of $W^*$ satisfy some inherent conditions like those
in \citet{Candes2011_JACM}, the problem
of practical matrix completion is generally ill-posed. This motivated us to
include the Tikhonov regularization term $\frac{\epsilon}{2}
\norm{\mathcal{P}_{\overline{\Omega}}(W)}^2$ in
\eqref{eq:MC_Formulation_0}, where $\overline{\Omega}$ denotes the
complement of $\Omega$, and $0<\epsilon<1 $ is a small constant.
Roughly speaking, what the regularization term does
is to pick the solution $W$ which has the smallest
$\norm{\mathcal{P}_{\overline{\Omega}}(W)}$ among all
the candidates in the optimal solution set of the
non-regularized problem. Notice that we only put a regularization
on those elements of $W$ in $\overline{\Omega}$ as we do not
wish to perturb those elements of $W$ in the fitting term.
Finally, with the Tikhonov regularization and the bound constraint on $\norm{E}$,
we can show that problem \eqref{eq:MC_Formulation_0} has a global
minimizer.

By defining $H\in\R^{m\times n}$ to be the matrix such that
\begin{eqnarray}
 H_{ij} = \left\{ \begin{array}{ll}
  1 & \mbox{if $(i,j)\in \Omega$} \\[5pt]
 \sqrt{\epsilon} &  \mbox{if $(i,j)\not\in \Omega$},
\end{array} \right.
\label{eq:H}
\end{eqnarray}
and those elements of $E$ and $\widehat{W}$ in $\overline{\Omega}$ to be zero, we can rewrite the objective function in \eqref{eq:MC_Formulation_0}
in a compact form, and the problem becomes:
\begin{eqnarray}\label{eq:MC_Formulation}
\begin{array}{rl}
 \underset{W,E}{\text{min}}\;
 & \frac{1}{2}\|H\circ(W+E-\widehat{W})\|^2
\\[5pt]
 \text{s.t.}\;
 & \mathrm{rank}(W)\leq r, \; W\in \R^{m\times n}
\\[5pt]
 & \|E\|_0\leq N_0, \;  \norm{E} \leq K_E, \; E\in \R^{m\times n}_\Omega.
\end{array}
\end{eqnarray}
 In the above,
the notation ``$\circ$" denotes the element-wise product between two matrices.

We propose PARSuMi, a proximal alternating minimization algorithm motivated by the algorithm in \citet{attouch2010proximal} to solve
\eqref{eq:MC_Formulation}. This involves solving two subproblems each with an auxiliary proximal regularization term.
It is important to emphasize that
the subproblems in our case are non-convex and hence it is essential to
design appropriate algorithms to solve the subproblems to global optimality, at
least empirically.
We develop essential reformulations of the subproblems and design novel techniques to efficiently solve each subproblem, provably achieving the global optimum for one, and empirically so for the other.
We also prove that our algorithm is guaranteed to converge to a limit point,
which is necessarily a stationary point
of (\ref{eq:MC_Formulation}). We emphasize here that the convergence is established even
though one of the subproblems may not be solved to global optimality.
Together with the initialization schemes we have designed based on the convex relaxation
of (\ref{eq:MC_Formulation}), our
method is able to solve challenging real matrix completion problems with
corruptions robustly and accurately.
As we demonstrate in the experiments, PARSuMi is able to provide excellent  reconstruction of unobserved feature trajectories in the classic Oxford Dinosaur sequence for SfM, despite structured (as opposed to random) observation pattern and data corruptions.
It is also able to  solve photometric stereo to high precision despite 
severe violations of the Lambertian model (which underlies the rank-3 factorization) due to shadow, highlight and facial expression difference. Compared to state-of-the-art methods such as GRASTA \citep{he2011grasta}, Wiberg~$\ell_1$ \citep{WibergL1} and BALM \citep{DelBue2012balm}, our results are substantially better both qualitatively and quantitatively.


Note that in (\ref{eq:MC_Formulation}) we do not seek convex relaxation of any form, but rather constrain the rank and the corrupted entries' cardinality directly in their original forms. While it is generally not possible to have an algorithm guaranteed to compute the global optimal solution, we demonstrate that with appropriate initializations,
the faithful representation of the original problem often offers significant advantage over the convex relaxation approach in denoising and corruption recovery, and is thus more successful in solving real problems.

The rest of the paper is organized as follows. In Section~\ref{sec:PrevWorks}, we provide a comprehensive review of the existing theories and algorithms for practical matrix completion,
summarizing the strengths and weaknesses of nuclear norm minimization
and matrix factorization. In Section~\ref{sec:numerical_MF}, we conduct numerical evaluations of predominant matrix factorization methods, revealing those algorithms that are less-likely to be trapped at local minima. Specifically, these features include parameterization on a subspace and second-order Newton-like iterations. Building upon these findings, we develop the PARSuMi scheme in Section~\ref{sec:PARSuMi} to simultaneously handle sparse corruptions, dense noise and missing data. The proof of convergence and a convex initialization scheme are also provided in this section.  In Section~\ref{sec:experiment}, the proposed method is evaluated on both synthetic and real data and is shown to outperform the current state-of-the-art algorithms for robust matrix completion.



\section{A survey of results}\label{sec:PrevWorks}
\subsection{Matrix completion and corruption recovery via nuclear norm minimization}\label{sec:MC_theory}
\begin{table*}[ht]
  \centering
  \begin{tabular}{|c|p{2cm}|p{2cm}|p{2cm}|p{2cm}|p{2cm}|p{2cm}|}
     \hline
                                & MC \citep{Candes_ExactMC} & RPCA \citep{Candes2011_JACM} & NoisyMC \citep{CandesNoise} & StableRPCA \citep{RPCA2010stable} & RMC \citep{li2013compressed} & RMC \citep{chen2011erasures}\\\hline
     Missing data               & Yes & Yes & Yes & No & Yes & Yes\\
     Corruptions                & No & Yes & No & Yes & Yes & Yes\\
     Noise                      & No & No & Yes & Yes & No & No\\
     Deterministic $\Omega$     &No & No & No & No & No & Yes\\
     Deterministic
     $\tilde{\Omega}$           & No & No & No & No & No & Yes\\
     \hline
   \end{tabular}
       \caption{Summary of the theoretical development for matrix completion and corruption recovery.}\label{tab:MC_theory}
\end{table*}
Recently, the most prominent approach for solving a matrix completion problem is via the following nuclear norm minimization:
\begin{equation}\label{eq:MC_nuc}
\underset{W}{\text{min}}\, \left\{
 \|W\|_* \, \middle|\,
 \mathcal{P}_{\Omega}(W-\widehat{W})=0 \right\},
\end{equation}
in which $\mathrm{rank}(X)$ is replaced by the nuclear norm $\|X\|_*=\sum_i\sigma_i(X)$,
where the latter is the tightest convex relaxation of $\mathrm{rank}$ over the
unit (spectral norm) ball. \citet{Candes_ExactMC} showed that when sampling is uniformly random and sufficiently dense, and the underlying low-rank subspace is \emph{incoherent} with respect to the standard bases, then the remaining entries of the matrix can be exactly recovered. The guarantee was later improved in \citet{candes2010optimal,recht2009simpler}, and extended for noisy data in \citet{CandesNoise,negahban2012restricted} relaxed the equality constraint to
$$ \|\mathcal{P}_{\Omega}(W-\widehat{W})\|\leq\delta. $$
Using similar assumptions and arguments, \citet{Candes2011_JACM} and \citet{Parrilo_RPCA} concurrently proposed solution to the related problem of robust principal component analysis (RPCA) where the low-rank matrix can be recovered from sparse corruptions (with no missing data\footnote{ \citet{Candes2011_JACM} actually considered missing data too, but their guarantee (Theorem~1.2) for \eqref{eq:RMC_nuc} is only preliminary according to their own remarks. A stronger result is released later by the same group in \citet{li2013compressed}.}). This is formulated as
\begin{equation}\label{eq:RPCA_nuc}
 \underset{W,E}{\text{min}}\,  \left\{
 \|W\|_*+\lambda\|E\|_1 \, \middle|\,
 W+E=\widehat{W} \right\}.
\end{equation}
Noisy extension and improvement of the guarantee for RPCA were provided by
\citet{RPCA2010stable} and \citet{ganesh2010dense} respectively.
\citet{chen2011erasures} and \citet{li2013compressed} combined \eqref{eq:MC_nuc} and \eqref{eq:RPCA_nuc} and provided guarantee for the following
\begin{equation}\label{eq:RMC_nuc}
 \underset{W,E}{\text{min}} \, \left\{
  \|W\|_*+\lambda\|E\|_1\,\middle|\,
 \mathcal{P}_{\Omega}(W+E-\widehat{W})=0
 \right\}.
\end{equation}
In particular, the results in \citet{chen2011erasures} lifted the uniform random support assumptions in previous works by laying out the exact recovery condition for a class of deterministic sampling ($\Omega$) and corruptions ($\tilde{\Omega}$) patterns.

We summarize the theoretical and algorithmic progress in  practical matrix completion achieved by each method in Table~\ref{tab:MC_theory}. It appears that researchers are moving towards analyzing all possible combinations of the problems; from past indication, it seems entirely plausible albeit tedious to show that the noisy extension
\begin{equation}\label{eq:RMCN_nuc}
 \underset{W,E}{\text{min}} \left\{
 \|W\|_*+\lambda\|E\|_1\,\middle|\,
 \|\mathcal{P}_{\Omega}(W+E-\widehat{W})\|\leq \delta
\right\}
\end{equation}
will return a solution stable around the desired $W$ and $E$ under appropriate assumptions. Wouldn't that solve the practical matrix completion problem altogether?

The answer is unfortunately no.
While this line of research have provided profound understanding of practical matrix completion itself, the actual performance of the convex surrogate on real problems (e.g., movie recommendation) is usually not competitive against nonconvex approaches such as matrix factorization. Although convex relaxation is amazingly equivalent to the original problem under certain conditions, those well versed in practical problems will know that
those theoretical conditions are usually not satisfied by real data.
Due to noise and model errors, real data are seldom truly low-rank (see the comments on Jester joke dataset in \citet{keshavan2009comparison}), nor are they as incoherent as randomly generated data.
More importantly, observations are often structured (e.g., diagonal band shape in SfM) and hence do not satisfy the random sampling assumption needed for the tight convex relaxation approach.
 As a consequence of all these factors, the recovered $W$ and $E$ by convex optimization are often neither low-rank nor sparse in practical matrix completion. This can be further explained by the so-called ``Robin Hood'' attribute of $\ell_1$ norm (analogously, nuclear norm is the $\ell_1$ norm in the spectral domain), that is, it tends to steal from the rich and give it to the poor, decreasing the inequity of ``wealth'' distribution. Illustrations of the attribute will be given in Section~\ref{sec:experiment}.

Nevertheless, the convex relaxation approach has the advantage that one can design
\emph{efficient} algorithms to find or approximately reach the \emph{global} optimal solution of the given convex formulation. In this paper, we take advantage of the convex relaxation approach and use it to provide a powerful initialization
for our algorithm to
converge to the correct solution.

\subsection{Matrix factorization and applications}

Another widely-used method to estimate missing data in a low-rank matrix is matrix factorization (MF). It is at first considered as a special case of the weighted low-rank approximation problem with $\{0,1\}$ weight by \citeauthor{gabriel1979lower} in \citeyear{gabriel1979lower} and much later by \citet{srebro2003weighted}. The buzz of Netflix Prize further popularizes the missing data problem as a standalone topic of research. Matrix factorization turns out to be a robust and efficient realization of the idea that people's preferences of movies are influenced by a small number of latent factors and has been used as a key component in almost all top-performing recommendation systems \citep{koren2009MF} including BellKor's Pragmatic Chaos, the winner of the Netflix Prize \citep{koren2009bellkor}.

In computer vision, matrix factorization with missing data is recognized as an important problem too. Tomasi-Kanade affine factorization \citep{Kanade_Factorization}, Sturm-Triggs projective factorization \citep{sturm1996factorization}, and many techniques in Non-Rigid SfM and motion tracking \citep{Paladini_NRSfM} can all be formulated as a matrix factorization problem. Missing data and corruptions emerge naturally due to occlusions and tracking errors. For a more exhaustive survey of computer vision problems that can be modelled by matrix factorization, we refer readers to \citet{DelBue2012balm}.


Regardless of its applications, the key idea is that when $W=UV^T$, one ensures that
the required rank constraint is satisfied by restricting the factors $U$ and $V$ to be
in $\R^{m\times r}$ and $\R^{n\times r}$ respectively. Since the $(U,V)$ parameterization has a much smaller degree of freedom than the dimension of $W$, completing the missing data becomes a better posed problem. This gives rise to the following optimization problem:
\begin{equation}\label{eq:MC_L2}
\begin{aligned}
& \underset{U,V}{\text{min}}
& & \frac{1}{2}\left \Vert \mathcal{P}_{\Omega}(UV^T - \widehat{W}) \right \Vert^2
\end{aligned}
\end{equation}
or its equivalence reformulation
\begin{equation}\label{eq:MC_L22}
\underset{U}{\text{min}}\,
\left\{ \frac{1}{2}\left \Vert \mathcal{P}_{\Omega}(UV(U)^T - \widehat{W}) \right \Vert^2
\middle| U^T U = I_r \right\}
\end{equation}
where the factor $V$ is now a function of $U$.


Unfortunately, \eqref{eq:MC_L2} is not a convex optimization problem. The quality of the solutions one may get by minimizing this objective function depends on specific algorithms and their initializations. Roughly speaking, the various algorithms for \eqref{eq:MC_L2} may be grouped into three categories: \textbf{alternating minimization}, \textbf{first order} gradient methods and \textbf{second order} Newton-like methods.


Simple approaches like alternating least squares (ALS) or equivalently PowerFactorization \citep{hartley2003powerfactorization} fall into the first category. They alternatingly fix one factor and minimize the objective over the other using least squares method. A more sophisticated algorithm is BALM \citep{DelBue2012balm}, which uses the Augmented Lagrange Multiplier method to gradually impose additional problem-specific manifold constraints. The inner loop however is still alternating minimization. This category of methods has the reputation of reducing the objective value quickly in the first few iterations, but they usually take a large number of iterations to converge to a high quality solution \citep{Damped_Newton_2005}.

First order gradient methods are efficient, easy to implement and they are able to scale up to million-by-million matrices if stochastic gradient descent is adopted. Therefore it is very popular for large-scale recommendation systems. Typical approaches include Simon Funk's incremental SVD \citep{funk2006netflix}, nonlinear conjugate gradient \citep{srebro2003weighted} and more sophisticatedly, gradient descent on the Grassmannian/Stiefel manifold, such as GROUSE \citep{balzano2010grouse} and OptManifold \citep{yin2013orthogonality}. These methods, however, as we will demonstrate later, easily get stuck in local minima\footnote{Our experiment on synthetic data shows that the strong Wolfe line search adopted by \citet{srebro2003weighted} and \citet{yin2013orthogonality} somewhat ameliorates the issue, though it does not seem to help much on real data.}.

The best performing class of methods are the second order Newton-like algorithms, in that they demonstrate superior performance in both accuracy and the speed of convergence (though each iteration requires more computation); hence they are suitable for small to medium scale problems requiring high accuracy solutions
(e.g., SfM and photometric stereo in computer vision). Representatives of these algorithms include the damped Newton method \citep{Damped_Newton_2005}, Wiberg($\ell_2$) \citep{WibergL2}, LM\_S and LM\_M of \citet{Subspace_ChenPei_2008} and LM\_GN,  which is a variant of LM\_M using Gauss-Newton (GN) to approximate the Hessian function.

As these methods are of special importance in developing our PARSuMi algorithm, we conduct extensive numerical evaluations of these algorithms in Section~\ref{sec:numerical_MF} to understand their pros and cons as well as the key factors that lead to some of them finding global optimal solutions more often than others.

It is worthwhile to note some delightful recent efforts to scale the first two classes of MF methods to internet scale, e.g., parallel coordinate descent extension for ALS \citep{yu2012scalable} and stochastic gradient methods in ``Hogwild'' \citep{Recht2011Hogwild}. It will be an interesting area of research to see if the ideas in these papers can be used to make the second order methods more scalable.


In addition, there are a few other works in each category that take into account the corruption problem by changing the quadratic penalty term of \eqref{eq:MC_L2} into $\ell_1$-norm or Huber function
\begin{equation}\label{eq:MC_L1}
\begin{aligned}
& \underset{U,V}{\text{min}}
& & \left \Vert \mathcal{P}_{\Omega}(UV^T - \widehat{W}) \right \Vert_1\, ,
\end{aligned}
\end{equation}
\begin{equation}\label{eq:MC_huber}
\begin{aligned}
& \underset{U,V}{\text{min}}
& &
\sum_{(ij)\in\Omega} \mbox{Huber}\big((UV^T - \widehat{W})_{ij} \big).
\end{aligned}
\end{equation}
Notable algorithms to solve these formulations include alternating linear programming (ALP) and alternating quadratic programming (AQP) in \citet{ke2005L1}, GRASTA \citep{he2011grasta} that extends GROUSE, as well as Wiberg~$\ell_1$ \citep{WibergL1} that uses a second order Wiberg-like iteration.
While it is well known that the $\ell_1$-norm or Huber penalty term can better handle
outliers, and the models (\ref{eq:MC_L1}) and (\ref{eq:MC_huber})
are seen to be effective in some problems, there is not much reason for a ``convex'' relaxation of the $\ell_0$ pseudo-norm\footnote{The cardinality of non-zero entries, which strictly speaking is not a norm.}, since the rank constraint imposed by matrix factorization is already highly non-convex. Empirically, we find that $\ell_1$-norm penalty offers poor denoising ability to dense noise and also suffers from ``Robin Hood'' attribute. Comparison with this class of methods will be given later in Section~\ref{sec:experiment}, which shows that our method can better handle noise and corruptions.

The practical advantage of $\ell_0$ over $\ell_1$ penalty is well illustrated in \citet{DRMF}, where \citeauthor{DRMF} proposed an $\ell_0$-based robust matrix factorization method which deals with corruptions and a given rank constraint. Our work is similar to \citet{DRMF} in that we both eschew the convex surrogate $\ell_1$-norm in favor of using the $\ell_0$-norm directly. However, our approach treats both corruptions and missing data. More importantly, our treatment of the problem is different and it results in a convergence guarantee that covers the algorithm of \citet{DRMF} as a special case; this will be further explained in Section~\ref{sec:PARSuMi}.

\begin{figure*}[ht]
  \centering
  \includegraphics[width=0.63\linewidth]{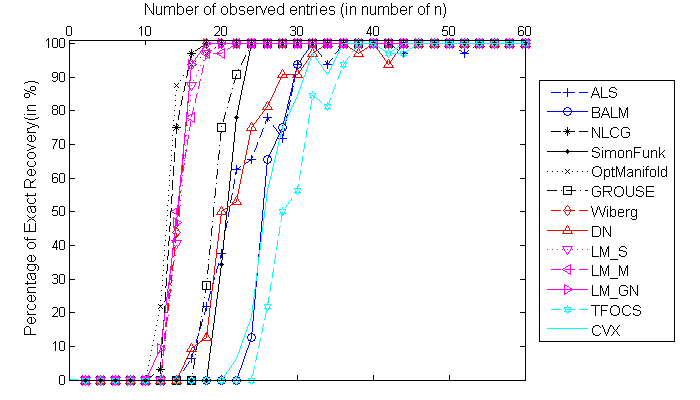}
  \includegraphics[width=0.35\linewidth]{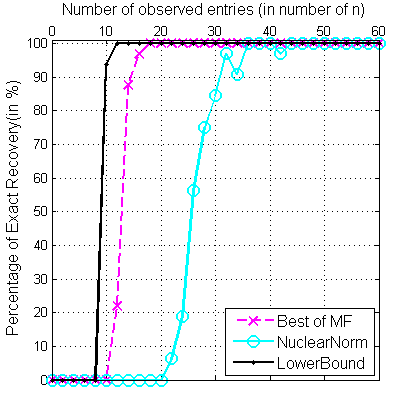}\\
  \caption{Exact recovery with \textbf{increasing number of random observations}. Algorithms (random initialization) are evaluated on 100 randomly generated rank-$4$ matrices of dimension $100\times100$. The number of observed entries increases from $0$ to $50n$. To account for small numerical error, the result is considered ``exact recovery'' if the RMSE of the recovered entries is smaller than $10^{-3}$. On the left, CVX \citep{cvx} and TFOCS \citep{tfocs}  (in cyan) solves the nuclear norm based matrix completion \eqref{eq:MC_nuc}, everything else aims to solve matrix factorization \eqref{eq:MC_L2}. On the right, the best solution of MF across all algorithms is compared to the CVX solver for nuclear norm minimization (solved with the highest numerical accuracy) and a lower bound (below the bound, the number of samples is smaller than $r$ for at least a row or a column).}\label{fig:mf_exp1}
\end{figure*}

\subsection{Emerging theory for matrix factorization}

As we mentioned earlier, a fundamental drawback of matrix factorization methods
for low rank matrix completion is the lack of proper theoretical foundation. However, thanks to the better understanding of low-rank structures nowadays, some theoretical analysis of this problem slowly emerges.
This class of methods are essentially designed for solving
noisy matrix completion problem with an explicit rank
constraint, i.e.,
\begin{equation}\label{eq:MC_rank_constraint}
 \underset{W}{\text{min}}
\left\{
 \frac{1}{2}\left \Vert \mathcal{P}_{\Omega}(W - \widehat{W}) \right \Vert^2
\,\middle|\, \mathrm{rank}(W)\leq r
\right\}.
\end{equation}
From a combinatorial-algebraic perspective, \citet{kiraly2012algebraic} provided a sufficient and necessary condition on the existence of an unique rank-$r$ solution to \eqref{eq:MC_rank_constraint}. It turns out that if the low-rank matrix is \emph{generic}, then
\emph{unique completability} depends only on the support of the observations $\Omega$.
This suggests that the incoherence and random sampling assumptions typically required by
various nuclear norm minimization methods may limit the portion of problems solvable by the latter to only a small subset of those solvable by matrix factorization methods.

Around the same time, \citet{wang2012stability} studied the stability of matrix factorization under arbitrary noise. They obtained a stability bound for the optimal solution of \eqref{eq:MC_rank_constraint} around the ground truth, which turns out to be
better than the corresponding bound for nuclear norm minimization in \citet{CandesNoise} by a scale of $\sqrt{\min{(m,n)}}$ (in Big-O sense). The study however bypassed the practical problem of how to obtain the global optimal solution for this non-convex problem.

This gap is partially closed by the recent work of  \citet{jain2012alt_min_global}, in which the global minimum of \eqref{eq:MC_rank_constraint} can be obtained up to an accuracy $\epsilon$ with $O(\log{1/\epsilon})$ iterations using a slight variation of the ALS scheme. The guarantee requires the observation to be noiseless, sampled uniformly at random and the underlying subspace of $W$ needs to be incoherent---basically all assumptions in the convex approach---yet still requires slightly more observations than that for nuclear norm minimization. It does not however touch on when the algorithm is able to find the global optimal solution when the data is noisy.
Despite not achieving stronger theoretical results nor under weaker assumptions than the
convex relaxation approach,
this is the first guarantee of its kind for matrix factorization. Given its more effective empirical performance, we believe that there is great room for improvement on the theoretical front. A secondary contribution of this paper is to find the potentially ``right'' algorithm or rather constituent elements of algorithm for theoreticians to look deeper into.


\section{Numerical evaluation of matrix factorization methods}\label{sec:numerical_MF}

To better understand the performance of different methods, we compare the following attributes quantitatively for all three categories of approaches that solve \eqref{eq:MC_L2} or
\eqref{eq:MC_L22}\footnote{As a reference, we also included nuclear norm minimization that solve \eqref{eq:MC_nuc} where applicable.}:
\begin{description}
  \item[\textbf{Sample complexity}]
Number of samples required for exact recovery of random uniformly sampled observations in random low-rank matrices, an index typically used to quantify the performance of nuclear norm based matrix completion.
  \item[\textbf{Hits on global optimal[synthetic]}] The proportion of random initializations that lead to the global optimal solution on random low rank matrices with (a) increasing Gaussian noise, (b) exponentially decaying singular values.
  \item[\textbf{Hits on global optimal[SfM]}] The proportion of random initializations that lead to the global optimal solution on the Oxford Dinosaur sequence \citep{Damped_Newton_2005} used in the SfM community.
\end{description}

\begin{figure}[tb]
  \centering
  \includegraphics[width=\linewidth,height=0.67\linewidth]{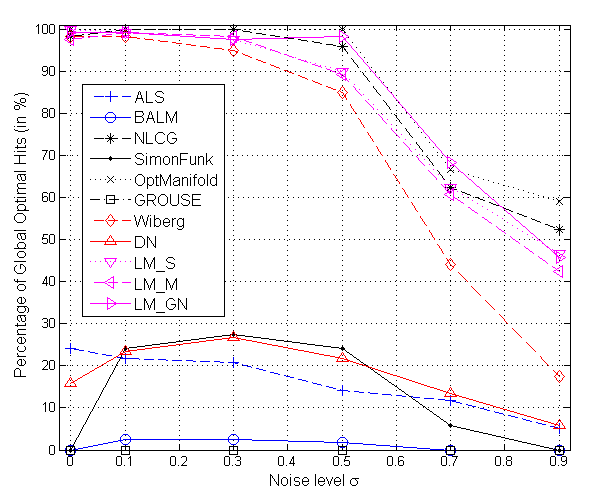} \\
  \caption{Percentage of hits on global optimal with \textbf{increasing level of noise}. Five rank-$4$ matrices are generated by multiplying two standard Gaussian matrices of dimension $40\times 4$ and $4\times 60$. 30\% of entries are uniformly picked as observations with additive Gaussian noise $N(0,\sigma)$. 24 different random initialization are tested for each matrix. The ``global optimal'' is assumed to be the solution with lowest objective value across all testing algorithm and all initializations.}\label{fig:mf_exp2a}
\end{figure}
\begin{figure}[tb]
  \centering
  \includegraphics[width=\linewidth, height=0.67\linewidth]{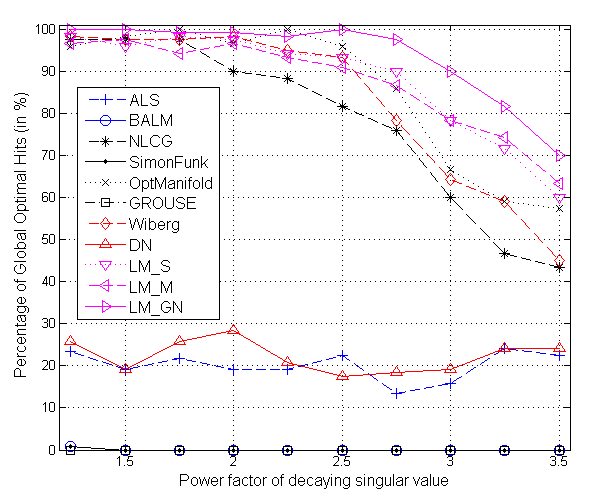}\\
  \caption{Percentage of hits on global optimal for \textbf{ill-conditioned low-rank matrices}. Data are generated in the same way as in Fig.~\ref{fig:mf_exp2a} with $\sigma=0.05$, except that we further take SVD and rescale the $i^{th}$ singular value according to $1/\alpha^i$. The Frobenious norm is normalized to be the same as the original low-rank matrix. The exponent $\alpha$ is given on the horizontal axis.}\label{fig:mf_exp2b}
\end{figure}
\begin{figure}[tb]
  \centering
  \includegraphics[width=\linewidth]{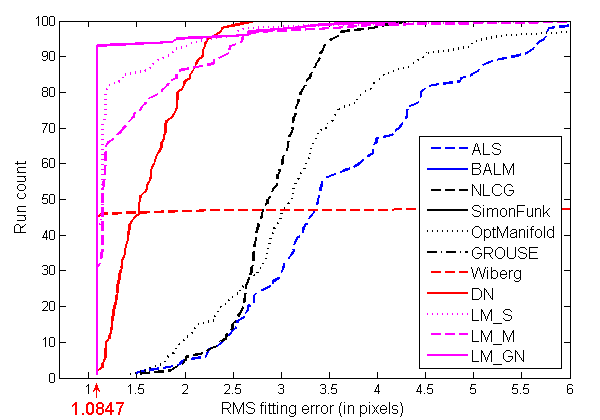}\\
  \caption{Cumulative histogram on the pixel RMSE for 100 randomly initialized runs
 conducted for each algorithm on Dinosaur sequence. The curve summarizes how many runs of each algorithm corresponds to the global optimal solution (with pixel RMSE 1.0847) on the horizontal axis. Note that the input pixel coordinates are normalized to between $[0,1]$ for the experiments, but to be comparable with \citet{Damped_Newton_2005}, the objective value is scaled back to the original size.}\label{fig:mf_exp3}
\end{figure}

%

The sample complexity experiment in Fig.~\ref{fig:mf_exp1} shows that the best performing matrix factorization algorithm attains exact recovery with the number of observed entries at roughly 18\%, while CVX for nuclear norm minimization needs roughly 36\% (even worse for numerical solvers such as TFOCS). This seems to imply that the sample requirement for MF is fundamentally smaller than that of nuclear norm minimization. As MF assumes known rank of the underlying matrix while nuclear norm methods do not, the results we observe are quite reasonable. In addition, among different MF algorithms, some perform much better than others. The best few of them achieve something close to the lower bound\footnote{The lower bound is given by the percentage of randomly generated data that have at least one column or row having less than $r$ samples. Clearly, having at least $r$ samples for every column and row is a necessary condition for exact recovery.}. This corroborates our intuition that MF is probably a better choice for problems with known rank.

From Fig.~\ref{fig:mf_exp2a}~and~\ref{fig:mf_exp2b}, we observe that the following classes of algorithms, including LM\_X series \citep{Subspace_ChenPei_2008}, Wiberg \citep{WibergL2}, Non-linear Conjugate Gradient method (NLCG) \citep{srebro2003weighted} and the curvilinear search on Stiefel manifold\\ (OptManifold \citep{yin2013orthogonality}) perform significantly better than others in reaching the global optimal solution despite their non-convexity. The percentage of global optimal hits from random initialization is promising even when the observations are highly noisy or when the condition number of the underlying matrix is very large\footnote{When $\alpha=3.5$ in Fig.~\ref{fig:mf_exp2b}, $r^{th}$ singular value is almost as small as the spectral norm of the input noise.}.

The common attribute of the four algorithms is that they are all based on the
model \eqref{eq:MC_L22} which parameterizes the factor $V$ as a function of $U$ and then optimizes over $U$ alone.
This parameterization essentially reduces the problem to finding the best subspace that fits the data. What is different between them is the way they update the solution in each iteration. OptManifold and NLCG adopt a Strong Wolfe line search that allows the algorithm to take large step sizes,
while the second order methods approximate each local neighborhood with a convex quadratic function and jump directly to the minimum of the approximation.
This difference 
appears to matter tremendously on the SfM experiment (see Fig.~\ref{fig:mf_exp3}). We observe that only the second order methods achieve global optimal solution frequently, whereas the Strong Wolfe line search adopted by both OptManifold and NLCG does not seem to help much on the real data experiment like it did in simulation with randomly generated data. Indeed, neither approach reaches the global optimal solution even once in the hundred runs, though they are rather close in quite a few runs. Despite these close runs, we remark that in applications like SfM, it is important to actually reach the global optimal solution. Due to the large amount of missing data in the matrix, even slight errors in the sampled entries can cause the recovered missing entries to go totally haywire with a seemingly good local minimum (see Fig.~\ref{fig:local_vs_global_min}). We thus refrain from giving any credit to local minima even if the $\mathrm{RMSE}_{\rm visible}$ error (defined in \eqref{eq:RMS_visible}) is very close to that of the global minimum.
\begin{equation} \label{eq:RMS_visible}
\mathrm{RMSE}_{\rm visible} := \frac{\Vert {\cal P}_{\Omega}(W_{\mathrm{recovered}} - \widehat{W}) \Vert}{\sqrt{|\Omega|}}.
\end{equation}
\begin{figure}[ht]
 \centering
 \subfigure[\scriptsize{Local minimum}]{
 \includegraphics[width=0.46\linewidth]{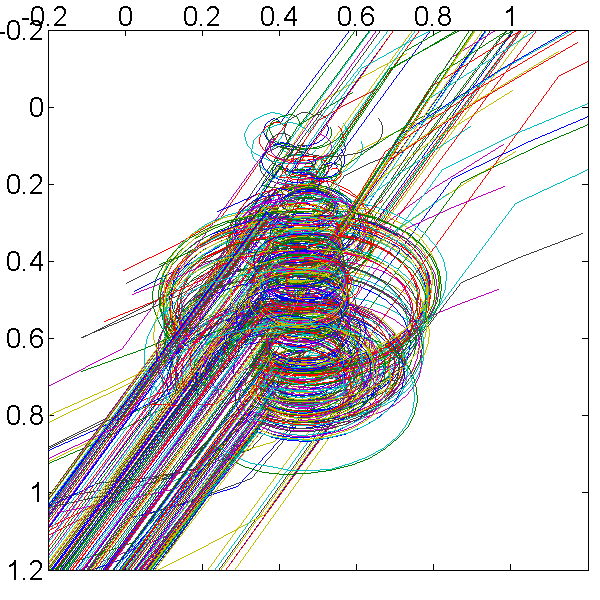}
 \label{fig:Traj_DN_local}
 }
 \subfigure[\scriptsize{Global minimum}]{
 \includegraphics[width=0.46\linewidth]{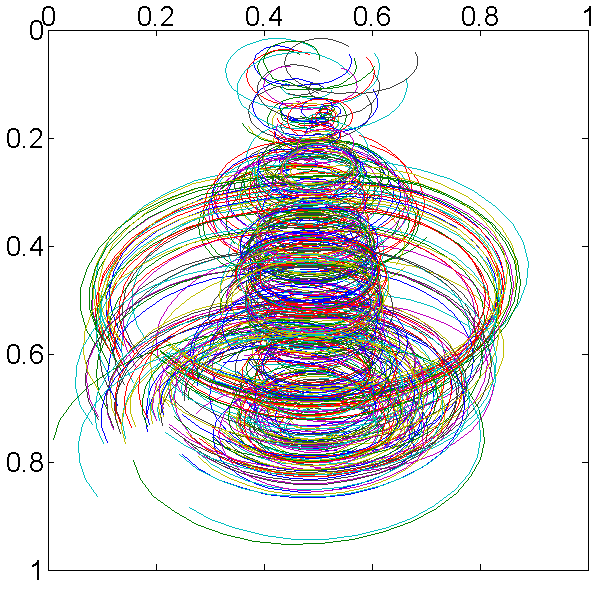}
 \label{fig:Traj_DN_global}
 }
 \caption{\small{Comparison of the feature trajectories
corresponding to a local minimum and global minimum of \eqref{eq:MC_L2}, given partial uncorrupted observations.
Note that $\mbox{RMSE}_{\rm visible}=1.1221$pixels in \subref{fig:Traj_DN_local} and  $\mbox{RMSE}_{\rm visible}=1.0847$pixels in \subref{fig:Traj_DN_global}. The latter is precisely the reported global minimum in \citet{WibergL2,Damped_Newton_2005} and \citet{Subspace_ChenPei_2008}. Despite the tiny difference in
$\mbox{RMSE}_{\rm visible}$, the filled-in values for missing data in \subref{fig:Traj_DN_local} are far off.}}
 \label{fig:local_vs_global_min}
\end{figure}

Another observation is that LM\_GN seems to work substantially better than other second-order methods with subspace or manifold parameterization, reaching global minimum 93 times out of the 100 runs. Compared to LM\_S and LM\_M, the only difference is the use of Gauss-Newton approximation of the Hessian. According to the analysis in \citet{chen2011hessian}, the Gauss-Newton Hessian provides the only non-negative convex quadratic approximation that preserves the so-called ``zero-on-$(n-1)$-D'' structure of a class of nonlinear least squares problems, into which \eqref{eq:MC_L2} can be formulated. Compared to the Wiberg algorithm that also uses Gauss-Newton approximation, the advantage of LM\_GN is arguably the better global convergence due to the augmentation of the LM damping factor.
Indeed, as we verify in the experiment, Wiberg algorithm fails to converge at all in most of its failure cases. The detailed comparisons of the second order methods and their running time on the Dinosaur sequence are summarized in Table~\ref{tab:MC_comparison}. Part of the results replicate that in \citet{Subspace_ChenPei_2008}; however, Wiberg algorithm and LM\_GN have not been explicitly compared previously. It is clear from the Table that LM\_GN is not only better at reaching the optimal solution, but also computationally cheaper than other methods which require explicit computation of the Hessian\footnote{Wiberg algoirthm takes longer time mainly because it sometimes does not converge and exhausts the maximum number of iterations.}.
\begin{table*}[ht]
\centering
\begin{tabular}{|l|l|l|l|l|l|l}
  \hline
   & DN & Wiberg & LM\textunderscore S & LM\textunderscore M & LM\textunderscore GN \\
   \hline
  No. of hits at global min. & 2 & 46 & 42 & 32 & 93 \\
  \hline
  No. of hits on stopping condition & 75 & 47 & 99 & 93 & 98\\
  \hline
  Average run time(sec) & 324 & 837 & 147 & 126 & 40 \\
  \hline
  No. of variables& (m+n)r & (m-r)r & mr & (m-r)r & (m-r)r \\
  \hline
  Hessian & Yes & Gauss-Newton & Yes & Yes & Gauss-Newton \\
  \hline
  LM/Trust Region & Yes & No & Yes & Yes & Yes \\
  \hline
  Largest Linear system to solve & $[(m+n)r]^2$ & $|\Omega|\times mr$ & $mr\times mr$ & $[(m-r)r]^2 $ & $[(m-r)r]^2$\\
  \hline
\end{tabular}
\caption{\small{Comparison of various second order matrix factorization algorithms}}
\label{tab:MC_comparison}
\end{table*}

To summarize the key findings of our experimental evaluation, we observe that: (a) the fixed-rank MF formulation requires less samples than nuclear norm minimization to achieve exact recovery; (b) the compact parameterization on the subspace, strong line search or second order update help MF algorithms in avoiding local minima in high noise, poorly conditioned matrix setting; (c) LM\_GN with Gauss-Newton update is able to reach the global minimum
with a very high success rate on a challenging real SfM data sequence.


\section{(Partially Majorized) Proximal Alternating Robust Subspace Minimization for \eqref{eq:MC_Formulation}}\label{sec:PARSuMi}
Our proposed PARSuMi method
for problem \eqref{eq:MC_Formulation}
works in two stages.
It first obtains a good initialization
from an efficient convex relaxation of \eqref{eq:MC_Formulation}, which will be
described in Section \ref{sec:Convex_relaxation}. This is followed by the minimization of the low rank matrix $W$ and the sparse matrix $E$ alternatingly until convergence. The efficiency of our PARSuMi method depends on the fact that the two inner minimizations of $W$ and $E$ admit efficient solutions, which will be derived in Sections \ref{sec:MC} and \ref{sec:outlier_removal} respectively. Specifically, in step~$k$, we compute $W^{k+1}$ from either
\begin{equation}\label{eq:MC_W}
\begin{aligned}
& \underset{W}{\text{min}}
\; \frac{1}{2}\|H\circ(W-\widehat{W}+E^k)\|^2
+  \frac{\beta_1}{2}\norm{H\circ(W-W^k)}^2\\
& \text{subject to}
\quad \mathrm{rank}(W)\leq r,
\end{aligned}
\end{equation}
or its quadratic majorization\footnote{We will explain this further shortly.}, and $E^{k+1}$ from
\begin{equation}\label{eq:MC_E}
\begin{aligned}
& \underset{E}{\text{min}}
\; \frac{1}{2}\norm{H\circ(W^{k+1}-\widehat{W}+E)}^2 +
\frac{\beta_2}{2}\norm{E-E^k}^2\\
& \text{subject to}
\quad \|E\|_0\leq N_0, \; \norm{E} \leq K_E,\; E\in \R^{m\times n}_\Omega,
\end{aligned}
\end{equation}
where $H$ is defined as in \eqref{eq:H}.
Note that the above iteration is different from applying a direct alternating minimization of \eqref{eq:MC_Formulation}. We have added the proximal regularization terms
$\norm{H\circ(W-W^k)}^2$ and $\|E-E^k\|^2$ to make the
objective functions in the subproblems coercive and hence ensuring that
$W^{k+1}$ and $E^{k+1}$ are well defined. As is shown in \citet{attouch2010proximal}, the
proximal terms are critical to ensure the critical point convergence of the sequence. 
In addition, we have added a quadratic majorization safeguard step when computing $W^{k+1}$. This is to safeguard the convergence even if our computed $W^{k+1}$ fails to be a global minimizer of \eqref{eq:MC_W}. 
Further details of the algorithm and the proof of its convergence are provided in the subsequent sections.


\subsection{Computation of $W^{k+1}$ in \eqref{eq:MC_W}}\label{sec:MC}

Our solution for \eqref{eq:MC_W} consists of two steps. We first transform the rank-constrained minimization \eqref{eq:MC_W} into an equivalent subspace fitting problem, then solve the new formulation using LM\_GN.

Motivated by the findings in Section \ref{sec:numerical_MF} where
the most successful algorithms for solving \eqref{eq:MC_rank_constraint}
are based on the formulation \eqref{eq:MC_L22}, we will  now derive a similar equivalent reformulation  of  \eqref{eq:MC_W}.
Our reformulation of \eqref{eq:MC_W} is motivated by
the $N$-parametrization of \eqref{eq:MC_rank_constraint} due to \citet{Subspace_ChenPei_2008}, who considered the task of matrix completion as finding the best subspace to fit the partially observed data. In particular, Chen proposes to solve \eqref{eq:MC_rank_constraint} using
\begin{equation} \label{eq:Subspace_MC_objective_function}
\underset{N}{\text{min}} \left\{
  \frac{1}{2}\sum\limits_i \hat{w}_i^T(I-\mathbb{P}_i)\hat{w}_i
\middle|  N^TN=I \right\}
\end{equation}
where $N$ is a $m\times r$ matrix whose column space is the underlying subspace to be reconstructed, $N_i$ is $N$ but with those rows corresponding to the missing entries in column $i$ removed. $\mathbb{P}_i=N_i N_i^+$
is the projection onto $\mathrm{span}(N_i)$ with $N_i^+$ being the Moore-Penrose
pseudo inverse of $N_i$, and the objective function minimizes the sum of squares distance between $\hat{w}_i$ to $\mathrm{span}(N_i)$, where $\hat{w}_i$ is the vector of observed entries in the $i^{th}$ column of $\widehat{W}$.

\subsubsection{N-parameterization of the subproblem \eqref{eq:MC_W}}

First define the matrix $\overline{H}\in \R^{m\times n}$ as follows:
\begin{eqnarray}\label{eq:H_bar}
 \overline{H}_{ij} = \left\{ \begin{array}{ll}
 \sqrt{1+\beta_1} & \mbox{if $(i,j)\in \Omega$}
  \\[5pt]
 \sqrt{ \epsilon+\epsilon \beta_1} &\mbox{if $(i,j)\not\in \Omega$.}
\end{array}\right.
\end{eqnarray}
Let $B^k\in \R^{m\times n}$ be the matrix defined by
\begin{eqnarray}\label{eq:Bk}
 B^k_{ij} = \left\{ \begin{array}{ll}
  \frac{1}{\sqrt{1+\beta_1}}(\widehat{W}_{ij} - E^k_{ij} + \beta_1 W^k_{ij}) & \mbox{if $(i,j)\in \Omega$}
  \\[5pt]
 \frac{\epsilon\beta_1}{\sqrt{\epsilon+\epsilon\beta_1}}\, W^k_{ij} &\mbox{if $(i,j)\not\in \Omega$.}
\end{array}\right.
\end{eqnarray}
Define the diagonal matrices $\mathbb{D}_{i} \in \R^{m\times m}$ to be
\begin{eqnarray}
   \mathbb{D}_{i} = {\rm diag}(\overline{H}_i), \quad i=1,\dots,n
\end{eqnarray}
where $\overline{H}_i$ is the $i$th column of $\overline{H}$.
It turns out that the $N$-parameterization for the regularized problem \eqref{eq:MC_W} has
a similar form as \eqref{eq:Subspace_MC_objective_function}, as shown below.

\begin{proposition}[Equivalence of subspace parameterization]\label{prop:equivalence}
Let $\mathbb{Q}_i(N) = \mathbb{D}_i N(N^T \mathbb{D}_i^2N)^{-1}N^T\mathbb{D}_i$,
which is the $m\times m$ projection matrix onto the column space of $\mathbb{D}_i N$. The problem \eqref{eq:MC_W} is
equivalent to the following problem:
\begin{equation} \label{eq:Subspace_MC_new}
\begin{aligned}
& \underset{N}{\text{min}}
\quad f(N) := \frac{1}{2}\sum_{i=1}^n \norm{B^k_i -\mathbb{Q}_i(N) B^k_i}^2\\
& \text{subject to}
\quad N^TN=I, \; N\in \R^{m\times r}
\end{aligned}
\end{equation}
where  $B^k_i$ is the $i$th columns of  $B^k$.
If $N_*$ is an optimal solution of \eqref{eq:Subspace_MC_new}, then $W^{k+1}$, whose
columns are defined
by
\begin{eqnarray}  \label{eq:W}
 W^{k+1}_i = \mathbb{D}_i^{-1}\mathbb{Q}_i(N_*) B^k_i,
\end{eqnarray}
is an optimal solution of \eqref{eq:MC_W}.
\end{proposition}
\begin{proof}
We can show by some algebraic manipulations that
the objective function in \eqref{eq:MC_W} is equal to
\begin{eqnarray*}
  \frac{1}{2}\norm{\overline{H}\circ W-B^k}^2
 + \mbox{constant}
\end{eqnarray*}
Now note that we have
\begin{eqnarray}
&&\{ W \in \R^{m\times n}\mid \mbox{rank}(W) \leq r \} \nonumber \\
&&=
\{ NC \mid N\in \R^{m\times r}, C \in \R^{r\times n}, N^TN= I\}.
\end{eqnarray}
Thus the problem \eqref{eq:MC_W} is equivalent to
\begin{eqnarray}
 \min_N \{ f(N) \mid  N^TN = I,  N\in \R^{m\times r}\}
\end{eqnarray}
where
\begin{eqnarray*}
 f(N) :=
 \min_{C}  \frac{1}{2}\norm{\overline{H}\circ(NC)-B^k}^2.
\end{eqnarray*}
To derive \eqref{eq:Subspace_MC_new} from the above, we need to obtain $f(N)$ explicitly as a function of $N$. For a given $N$, the unconstrained minimization problem over $C$ in $f(N)$
has a strictly convex objective function in $C$, and hence the unique global minimizer
satisfies the following optimality condition:
\begin{eqnarray}
 N^T ((\overline{H}\circ\overline{H})\circ(NC)) =
N^T (\overline{H}\circ B^k).
\end{eqnarray}
By considering the $i$th column $C_i$ of $C$, we get
\begin{eqnarray}
 N^T \mathbb{D}_i^2 N C_i = N^T \mathbb{D}_i B^k_i, \quad i=1,\dots,n.
\end{eqnarray}
Since $N$ has full column rank and $D^i$ is positive definite, the coefficient
matrix in the above equation is nonsingular, and hence
\begin{eqnarray*}
 C_i = (N^T \mathbb{D}_i^2 N)^{-1} N^T \mathbb{D}_i B^k_i.
\end{eqnarray*}
Now with the optimal $C_i$ above for the given $N$, we can show after some algebra
manipulations that
$f(N)$ is given as in \eqref{eq:Subspace_MC_new}.
 \qed
\end{proof}

We can see that when $\beta_1\downarrow 0$ in \eqref{eq:Subspace_MC_new},
then  the problem reduces to
\eqref{eq:Subspace_MC_objective_function},
with the latter's $\hat{w}_i$ appropriately modified to take into account of
$E^k$. Also, from the above proof, we see that the $N$-parameterization reduces the feasible region of $W$ by restricting $W$ to only those potential optimal solutions among the set of $W$
satisfying the expression in \eqref{eq:W}.
This seems to imply that it is not only equivalent but also advantageous to optimize over $N$ instead of $W$. While we have no theoretical justification of this conjecture,
it is consistent with our experiments in Section~\ref{sec:numerical_MF} which show
the superior performance of those algorithms using subspace parameterization in
finding global minima and vindicates the design motivations of the series of LM\_X algorithms in \citet{Subspace_ChenPei_2008}.

\subsubsection{LM\_GN updates}

Now that we have shown how to handle the regularization term and validated the equivalence of the transformation, the steps to solve \eqref{eq:MC_W} essentially generalize those of LM\_GN (available in Section~3.2 and Appendix~A of \citet{chen2011hessian}) to account for the general mask $H$. The derivations of the key formulae and their meanings are given in this section.


In general, Levenberg-Marquadt solves the non-linear problem with the following sum-of-squares objective function
\begin{equation}\label{eq:nonlinear_LS}
  \mathcal{L}(x)= \frac{1}{2}\sum_{i=1:n}\norm{y_i-f_i(x)}^2,
\end{equation}
by iteratively updating $x$ as follows:
$$x\leftarrow x+(J^TJ+\lambda I)^{-1}J^T\mathbf{r},$$
where $J=[J_1;\dots;J_n]$ is the Jacobian matrix and $J_i$ is the Jacobian matrix
of $f_i$; $\mathbf{r}$ is the concatenated vector of residual $r_i := y_i-f_i(x)$ for all $i$, and $\lambda$ is the damping factor that interpolates between Gauss-Newton update and gradient descent. We may also interpret the iteration as a Damped Newton method with a first order approximation of the Hessian matrix using $J^TJ$. 

Note that the objective function of \eqref{eq:Subspace_MC_new} can be expressed in the form of \eqref{eq:nonlinear_LS} by taking $x:=\mathrm{vec}(N)$, data $y_i:=B^k_i$, and function
$$f_i(x:=\mathrm{vec}(N))=\QQ_i (N) B^k_i = \QQ_iy_i$$

\begin{proposition} \label{prop:Jacobian}
Let $\mathcal{T} \in \R^{mr\times mr}$ be the permutation matrix such
that ${\rm vec}(X^T) = \mathcal{T} {\rm vec}(X)$ for any $X\in\R^{m\times r}$.
The Jacobian of $f_i(x) = \mathbb{Q}_i(N) y_i$ is given as follows:
\begin{eqnarray}
  J_i (x) =(\AA_i^T y_i)^T \otimes ((I-\mathbb{Q}_i)\DD_i)
+\; [ (\DD_i r_i)^T\otimes \AA_i] \mathcal{T}. \quad\;
\label{eq:Ji}
\end{eqnarray}
Also $J^TJ = \sum_{i=1}^n J_i^TJ_i$, $J^Tr = \sum_{i=1}^n J_i^T r_i$, where
\begin{eqnarray}
 J_i^T J_i &=& (\AA_i^T y_iy_i^T \AA_i) \otimes (\DD_i(I-\QQ_i)\DD_i)
  \nonumber \\[5pt]
&&+\, \mathcal{T}^T [ (\DD_i r_i r_i^T \DD_i) \otimes (\AA_i^T \AA_i)]
\mathcal{T}  \label{eq:JTJ}
\\[5pt]
J_i^Tr_i &=& {\rm vec}(\DD_i r_i (\AA_i^T y_i)^T).
\label{eq:JTr}
\end{eqnarray}
In the above, $\otimes$ denotes the Kronecker product.
\end{proposition}
\begin{proof}
Let $\AA_i = \DD_i N (N^T\DD_i^2 N)^{-1}$. Given sufficiently small $\dN$, we
can show that
 the directional derivative of $f_i$ at $N$ along $\dN$ is given by
\begin{eqnarray*}
  f_i^\prime (N+\dN) &= &
 (I-\mathbb{Q}_i)\DD_i\dN \AA_i^T y_i
+ \AA_i \dN^T \DD_i  r_i.
\end{eqnarray*}
By using the  property that
${\rm vec}(AXB)=(B^T\otimes A){\rm vec}(X)$,
we have
\begin{eqnarray*}
{\rm vec}(f_i^\prime (N+\dN)) &=&
 [(\AA_i^T y_i)^T \otimes ((I-\mathbb{Q}_i)\DD_i)] {\rm vec}(\dN)
\\
&&+ [ (\DD_i r_i)^T\otimes \AA_i] {\rm vec}(\dN^T)
\end{eqnarray*}
From here, the required result in \eqref{eq:Ji} follows.

To prove \eqref{eq:JTJ}, we make use of the following properties of Kronecker product:
$(A\otimes B)(C\otimes D) = (AC)\otimes (BD)$ and $(A\otimes B)^T = A^T\otimes B^T$.
By using these properties, we see that $J_i^T J_i$ has four terms, with two of the
terms containing involving $\DD_i(I-\QQ_i)\AA_i$ or its transpose.
But we can verify that $\QQ_i\AA_i=\AA_i$ and hence those two terms become
$0$. The remaining two terms are those appearing in \eqref{eq:JTJ} after
using the fact that $(I-\QQ_i)^2 = I-\QQ_i$.
Next we prove \eqref{eq:JTr}. We have
\begin{eqnarray*}
 J_i^T r_i = {\rm vec}(\DD_i (I-\QQ_i)r_i (\AA_i^Ty_i)^T) +
\mathcal{T}^T {\rm vec}(\AA_i^T r_i r_i^T \DD_i).
\end{eqnarray*}
By noting that $\AA_i^T r_i = 0$ and $ \QQ_i r_i =0$, we get
the required result in \eqref{eq:JTr}.
\qed
\end{proof}
The complete procedure of solving \eqref{eq:MC_W} is summarized in Algorithm~\ref{alg:LM_GN}. In all our experiments, the initial $\lambda$ is chosen as $1e-6$ and $\rho=10$.
\begin{algorithm}[tb]
   \caption{Levenberg-Marquadt method for \eqref{eq:MC_W}}
   \label{alg:LM_GN}
\begin{algorithmic}
   \STATE {\bfseries Input:} $\widehat{W}, E^k, W^k,\bar{H}$, objective function $\mathcal{L}(x)$ and initial $N^k$;  numerical parameter $\lambda$, $\rho>1$.
   \STATE{\bfseries Initialization:} Compute $y_i=B_i^k$ for $i=1,...,n$, and $x^0=\mathrm{vec}(N^k)$, $j=0.$
   \WHILE{not converged}
   \STATE{1. } Compute  $J^T\mathbf{r}$ and $J^TJ$ using \eqref{eq:JTr} and\eqref{eq:JTJ}.
   \STATE{2. } Compute $\Delta x=(J^TJ+\lambda I)^{-1}J^Tr$
              \WHILE{$\mathcal{L}(x+\Delta x)<\mathcal{L}(x)$}
                \STATE{(1)} $\lambda=\rho\lambda.$
                \STATE{(2)} $\Delta x=(J^TJ+\lambda I)^{-1}J^Tr.$
              \ENDWHILE
   \STATE{3.} $\lambda=\lambda/\rho.$
   \STATE{4. } Orthogonalize $N=\mathrm{orth}[\mathrm{reshape}(x^j+\Delta x)]$.
   \STATE{5. } Update $x^{j+1}=\mathrm{vec}(N)$.
   \STATE{6. } Iterate $j=j+1$
   \ENDWHILE
   \STATE {\bfseries Output:} $N^{k+1}=N$, $W^{k+1}$ using \eqref{eq:W}
with $N^{k+1}$ replacing $N_*$.
\end{algorithmic}
\end{algorithm}


\subsection{Quadratic majorization of \eqref{eq:MC_W}}\label{sec:majorized_MC}
Recall that while the LM$\_$GN  method may be highly successful in computing a global minimizer for \eqref{eq:MC_W} empirically, $W^{k+1}$ may fail to be a global minimizer occasionally. Thus before deriving the update rule for $E^{k+1}$, we consider minimizing a quadratic majorization of \eqref{eq:MC_W} as a safeguard step to ensure the convergence of the PARSuMi iterations.
Recall that \eqref{eq:MC_W} is equivalent to
$$W^{k+1}=\argmin_{W}\left\{\frac{1}{2}\|\overline{H}\circ (W-\hat{B}^k)\|^2\middle| \mathrm{rank}(W)\leq r\right\}$$
where $\hat{B}^k=\overline{H}^{-1}\circ B^k$ with $\overline{H}$  and $B^k$ defined as in \eqref{eq:H_bar} and  \eqref{eq:Bk} respectively.

For convenience, we denote the above objective function $F(W,\hat{B}^k)$.
Suppose that we have positive vectors $p\in\R^m$ and $q\in \R^n$
such that
\begin{eqnarray}
  \bar{H}_{ij}^2 \leq p_i q_j\quad \forall\; i=1,\dots,m, \; j=1,\dots,n.
\label{eq-1}
\end{eqnarray}
Note that the above inequality always holds if $p$, $q$ are chosen to be
\begin{eqnarray}
  p_i  &=& \max\{ \bar{H}_{ij} \mid j=1,\dots,n\}, \quad i=1,\dots,m
\nonumber \\
 q_j &=& \max\{ \bar{H}_{ij} \mid i=1,\dots,m\}, \quad j=1,\dots,n.
\label{eq-2}
\end{eqnarray}
Let $G^k=\nabla_W F(W^k,\hat{B}^k)$.
We majorize $F(W,\hat{B}^k)$ by bounding its Taylor expansion at $W^k$
\begin{align}
&F(W,\hat{B}^k) - F(W^k,\hat{B}^k)
=  \inprod{G^k}{W-W^k} \nonumber\\
&+ \frac{1}{2}\inprod{W-\hat{B}^k}{(\bar{H}\circ\bar{H})\circ (W-\hat{B}^k)}
\nonumber \\\
\leq &\;
 \inprod{G^k}{W-W^k} +  \frac{1}{2}\inprod{W-W^k}{P(W-W^k)Q}
\nonumber \\
=&
 \frac{1}{2}\norm{P^{1/2} W Q^{1/2} -U^k}^2
 - \frac{1}{2}\norm{ P^{-1/2}G^k Q^{-1/2}}^2  \label{eq:F_majorization}
\end{align}
where $P= \diag{p}$ and $Q = \diag{q}$,
$U^k=P^{1/2}W^k Q^{1/2} - P^{-1/2}G^k Q^{-1/2}$.

\begin{proposition}\label{prop:majorization_update}
The  minimizer of quadratic majorization function in
\eqref{eq:F_majorization} is given in closed-form by
\begin{equation}\label{eq:majorization_update}
  W_{\rm QM}^{k+1} \in P^{-1/2}\Pi_r(U^k)Q^{-1/2}.
\end{equation}
Here $\Pi_r(U^k)$ denotes the set of best rank-r approximation of $U^k$.
\end{proposition}
The proof is simple and is given in the Appendix. Note that this is a nonconvex minimization, yet we have an efficient closed-form solution thanks to SVD.

As we shall see later in Algorithm \ref{alg:PAM}, the global minimizer $W^{k+1}_{\rm QM}$
in
\eqref{eq:majorization_update}
of the quadratic
majorization function of $F(W,\hat{B}^k)$
is used as a safeguard when the computed solution $W^{k+1}$ from
\eqref{eq:MC_W} is inferior (which necessarily implies that
$W^{k+1}$ is not a global optimal solution) to $W^{k+1}_{\rm QM}$.
By doing so, the convergence of PARSuMi can be ensured, as we shall
prove in Theorem \ref{thm:critical_pt}.

\subsection{Sparse corruption recovery step \eqref{eq:MC_E}}\label{sec:outlier_removal}

In the sparse corruption step, we need to solve the $\ell_0$-constrained least squares minimization \eqref{eq:MC_E}. This problem is combinatorial in nature, but
fortunately, for our problem, we show that a closed-form solution can be obtained.
Let $x := \mathcal{P}_\Omega(E)$.
Observe that \eqref{eq:MC_E} can be expressed in the following equivalent
form:
\begin{eqnarray}
 \min_x \Big\{ \norm{x- b}^2  \mid \norm{x}_0 \leq N_0, \; \norm{x}^2 -K_E^2\leq 0\Big\}
\label{eq:MC_E2}
\end{eqnarray}
where
$b = \mathcal{P}_\Omega(\widehat{W}-W^{k+1}+\beta_2 E^k)/(1+\beta_2)$.

\begin{proposition}\label{prop:E_update} Let $I$ be the set of indices of the $N_0$ largest (in magnitude)
component of $b$. Then the nonzero components of the optimal solution $x$ of \eqref{eq:MC_E2} is given by
\begin{eqnarray}
 x_I = \left\{ \begin{array}{ll}
  K_E b_I/\norm{b_I} & \mbox{if $\norm{b_I} > K_E$}
\\[5pt]
  b_I & \mbox{if $\norm{b_I}\leq K_E$.}
\end{array} \right.
\label{eq:MC_E2x}
\end{eqnarray}
\end{proposition}
The proof (deferred to the Appendix) involves checking the optimality conditions of \eqref{eq:MC_E2} assuming known support set and finding the optimal support set in a decoupled fashion.

The procedure to obtain the optimal solution of \eqref{eq:MC_E} is summarized in Algorithm~\ref{alg:MC_E}. We remark that this is a very special case of $\ell_0$-constrained optimization; the availability of the exact closed form solution depends on both terms in \eqref{eq:MC_E} being decomposable into individual $(i,j)$ term. In general, if we change the
operator $M \rightarrow H\circ M$ in \eqref{eq:MC_E} to a general linear transformation (e.g., a sensing matrix in compressive sensing), or change the norm  $\|\cdot\|$ of the proximal term to some other norm such as spectral norm or nuclear norm, then the problem becomes NP-hard.
\begin{algorithm}[tb]
   \caption{Closed-form solution of \eqref{eq:MC_E}}
   \label{alg:MC_E}
\begin{algorithmic}
   \STATE {\bfseries Input:}$\widehat{W}, W^{k+1},E^k, \Omega$.
   \STATE{1. } Compute $b$ using \eqref{eq:MC_E2}.
   \STATE{2. } Compute $x$ using \eqref{eq:MC_E2x}.
   \STATE {\bfseries Output:}  $E^{k+1} = \cP_\Omega^*(x)$.
\end{algorithmic}
\end{algorithm}

\subsection{Algorithm}
Our method is summarized in Algorithm~\ref{alg:PARSuMi}. 
\begin{algorithm}[tb]
   \caption{(Partially Majorized) Proximal Alternating Robust Subspace Minimization (PARSuMi)}
   \label{alg:PARSuMi}
\begin{algorithmic}
   \STATE {\bfseries Input:}Observed data $\widehat{W}$, sample mask ${\Omega}$, parameter $r, N_0$. Initialization $W^0$ and $E^0$ (typically by Algorithm~\ref{alg:APG} described in Section~\ref{sec:Convex_relaxation}), $k=0$.
   \REPEAT
   \STATE{1a. } Solve \eqref{eq:MC_W} using Algorithm~\ref{alg:LM_GN} with $W^k$,$E^k$,$N^k$, obtain updates $\tilde{W}^{k+1}$ and $\tilde{N}^{k+1}$
   \STATE{1b. } Evaluate \eqref{eq:majorization_update} with $W^k$,$E^k$ obtain updates $\hat{W}^{k+1}$.
   \STATE{2. }  Assign $W^{k+1}=\argmin_{W\in\{\tilde{W}^{k+1},\hat{W}^{k+1}\}} F(W,\hat{B}^k)$, and then assign the corresponding $N^{(k+1)}$.
   \STATE{3. } Solve \eqref{eq:MC_E} using Algorithm~\ref{alg:MC_E} with $W^{k+1}, E^k$; obtain updates $E^{k+1}$.
   \UNTIL{$\|W^{k+1}-W^{k}\| < \|W^{k}\|\cdot10^{-6}$ and $\|E^{k+1}-E^{k}\| < \|E^{k}\|\cdot10^{-6}$}  \STATE {\bfseries Output:}
Accumulation points $\overline{W}$ and $\overline{E}$
\end{algorithmic}
\end{algorithm}
 Note that we do not need to know the exact cardinality of the corrupted entries; $N_0$ can be taken as an upper bound of the allowable number of corruptions. As a rule of thumb, 10-15\% of $|\Omega|$ is a reasonable size. The surplus in $N_0$ will only label a few noisy samples as corruptions, which should not affect the recovery of either $W$ or $E$, so long as the remaining $|\Omega|-N_0$ samples are still sufficient. The other parameter $r$ is typically given by the physical model of the problem. For those problems where $r$ is not known, choosing $r$ is analogous to choosing the regularization parameter as in other machine learning tasks. A large $r$ will lead to overfitting and poorly estimated missing data while an overly small $r$ will cause underfitting of the observed data.

\subsection{Convergence to a critical point }\label{sec:convergence}
In this section, we show the convergence
of Algorithm \ref{alg:PARSuMi} to a critical point.

Note that due to the non-convex nature of the subproblem \eqref{eq:MC_W}, Algorithm~\ref{alg:LM_GN} is guaranteed to converge only to a local minimum. Therefore, the result in \citet{attouch2010proximal} that requires global optimal solutions in all subproblems cannot be directly applied in our case for the critical point convergence proof. Empirically, we cannot hope LM\_GN to always find the global optimal solution of \eqref{eq:MC_W} either, as our experiments on LM\_GN in Section~\ref{sec:numerical_MF} clearly demonstrated. As a result, we design the partial majorization (Step 1b) in Algorithm~\ref{alg:PARSuMi} to safeguard against the case when the computed solution from Step 1a is not a global optimal solution. The safeguard step is powerful in that we do not need to assume anything on the computed solution of the subproblem  before we can prove the overall critical point convergence.



We start our convergence proof by first defining an equivalent formulation of \eqref{eq:MC_Formulation} in terms of closed, bounded sets. The convergence proof is then based on the indicator functions for these closed and bounded sets, which have the key lower semicontinuous property.

Let $K_W = 2\norm{\widehat{W}} + K_E$.
Define the closed and bounded sets:
\begin{eqnarray*}
  \mathcal{W} &=& \{ W \in \R^{m\times n} \mid {\rm rank}(W)\leq r, \norm{H\circ W} \leq
 K_W \}
\\[5pt]
\mathcal{E} &=& \{ E\in \R^{m\times n}_\Omega \mid
\norm{E}_0 \leq N_0, \norm{E} \leq K_E\}.
\end{eqnarray*}
We will first show that \eqref{eq:MC_Formulation} is equivalent to the problem given
in the next proposition.
\begin{proposition}
Let $f(W,E):=\frac{1}{2}\norm{H\circ(W+E-\widehat{W})}^2$.
The problem \eqref{eq:MC_Formulation} is equivalent  to the following problem:
\begin{eqnarray}
 \min \big\{ f(W,E)
\mid
  W\in\mathcal{W}, E\in \mathcal{E} \big\}.
\label{eq:MC_Formulation_2}
\end{eqnarray}
\end{proposition}

\begin{proof}
Observe that the only difference between \eqref{eq:MC_Formulation} and \eqref{eq:MC_Formulation_2} is the inclusion of the bound constraint on
$\norm{H\circ W}$ in \eqref{eq:MC_Formulation_2}. To show the equivalence, we only
need to show that any minimizer $(\overline{W}, \overline{E})$  of \eqref{eq:MC_Formulation} must
satisfy the bound constraint in $\mathcal{W}$. By definition, we know that
$$
 f(\overline{W}, \overline{E}) \leq f(0,0) = \frac{1}{2}\norm{\widehat{W}}^2.
$$
Now for any $(W,E)$ such that ${\rm rank}(W)\leq r$,
$E\in \mathcal{E}$ and $\norm{H\circ W} > K_W$, we must have
\begin{eqnarray*}
 & & \hspace{-0.7cm}
\norm{H\circ(W+E-\widehat{W})} \geq \norm{H\circ W} - \norm{H\circ(E-\widehat{W})}
\\[5pt]
& > & K_W - \norm{E}-\norm{\widehat{W}} \geq  \norm{\widehat{W}}.
\end{eqnarray*}
Hence $ f(W,E) > \frac{1}{2} \norm{\widehat{W}}^2 = f(0,0).$
This implies that we must have $\norm{H\circ \overline(W)} \leq K_W$.
\qed
\end{proof}
To establish the convergence of PARSuMi, it is more convenient for us
to consider the following generic problem which includes
\eqref{eq:MC_Formulation_2} as a special case.
Let $\mathcal{X}$, $\mathcal{Y}$, $\mathcal{Z}$  be  finite-dimensional
inner product spaces, and $f:\mathcal{X} \rightarrow \mathbb{R} \cup \{\infty\}$,
$g:\mathcal{Y} \rightarrow \mathbb{R} \cup \{\infty\}$ are lower semi-continuous functions.
We consider the problem:
\begin{align}\label{eq:L}
\min_{x,y} \{L(x,y) := f(x) + g(y) + q(x,y)\}
\end{align}
where
$
 q(x,y) = \frac{1}{2} \norm{Ax + By - c}^2
$
and $A:\mathcal{X}\rightarrow \mathcal{Z}$, $B: \mathcal{Y}\rightarrow \mathcal{Z}$ are
given linear maps.
For \eqref{eq:MC_Formulation_2}, we have
$\mathcal{X}=\mathcal{Z}=\R^{m\times n}$, $\mathcal{Y}=\R^{m\times n}_\Omega$,
$A(x) = H\circ x$, $B(y) = H\circ y$, $c=\widehat{W}$.
and $f$, $g$  are the following indicator functions,
\begin{eqnarray*}\label{eq:Sr}
f(x) = \left\{
\begin{array}{ll}
  0 & \quad \mbox{if $x\in\mathcal{W}$}\\
  \infty & \quad \mbox{otherwise}
\end{array} \right.
\quad
g(y)=   \left\{
\begin{array}{ll}
  0 & \quad \mbox{if $y \in\mathcal{E}$}\\
  \infty & \quad \mbox{otherwise}
 \end{array} \right.
\end{eqnarray*}
Note that in this case, $f$ are $g$ are lower semicontinuous since
indicator functions of closed sets are lower semicontinuous \citep{Rudin}.

To denote the corresponding majorization safeguard, we define, for a fixed $(\hx,\hy)$ and given  $M \succ A^*A$
\begin{align}
   Q(x; \hx,\hy) &:= q(\hx,\hy) + \inprod{\nabla_xq(\hx,\hy)}{x-\hx}
+ \frac{1}{2}
\norm{x-\hx}_M^2 
\label{eq-Q}
\\
\widehat{L}(x;\hx,\hy) &:=  Q(x; \hx,\hy) + f(x) + g(\hy) 
\label{eq-hL}
\end{align}
where $\norm{\cdot}_M$ is defined in the last part of Algorithm~\ref{alg:PAM}. 

Then, we have that
\begin{align}
  q(x,\hy) &= Q(x;\hx,\hy) -\frac{1}{2}\norm{x-\hx}_{M-A^*A}^2
\label{eq-qQ}
\\[5pt]
L(x,\hy) &= \widehat{L}(x;\hx,\hy) -\frac{1}{2}\norm{x-\hx}_{M-A^*A}^2\quad
\label{eq-LhL}
\end{align}

Consider the partially majorized proximal alternating minimization (PMPAM) outlined in Algorithm~\ref{alg:PAM}, which we have modified from  \citet{attouch2010proximal}. The algorithm alternates between minimizing $x$ and $y$,
but with the important addition of the quadratic Moreau-Yoshida regularization term (which is
also known as the proximal term) in each step. The
importance of  Moreau-Yoshida regularization for convex matrix optimization problems
has been demonstrated and studied in \citet{Nuclear_PPA,LogDet_PPA,Knorm_PPA}.
For our non-convex, non-smooth setting here, the importance of the proximal  term will become clear
when we prove the convergence of Algorithm \ref{alg:PAM}. For our problem \eqref{eq:MC_Formulation_2},
the positive linear maps $S$ and $T$ in Algorithm \ref{alg:PAM} correspond to $\beta_1 (H \circ H) \circ$ and $\beta_2 I$ respectively, where $\beta_1,\beta_2$ are
given positive parameters.
Our algorithm differs from that in \citet{attouch2010proximal}
by having the safeguard step (in Step 1b and Step 2) to ensure that
critical point convergence can be achieved even if the
computed solution in Step 1a is not globally optimal.
Observe that one can bypass Step 1a in Algorithm \ref{alg:PAM} completely and
always choose to use Step 1b. But the minimization in Step 1b
based on quadratic majorization may not reduce the
merit function $L(x,y_k)+\frac{1}{2}\norm{x-x_k}_S^2$
as quickly as the minimization in Step 1a. Thus it is necessary to have Step 1a
to ensure that Algorithm 4 converges at  a
reasonable speed.
We note that a
similar safeguard step can be introduced for the subproblem in Step 3
if the global optimality of $y_{k+1}$  is not guaranteed.


\begin{algorithm}[htb]
   \caption{Partially majorized proximal alternating minimization (PMPAM)}
   \label{alg:PAM}
\begin{algorithmic}
   \STATE {\bfseries Input:}$(x_0,y_0)\in \cX\times\cY$;
positive linear operators $S$ and $T$. Choose $M$ such that
$M\succ A^*A + S$ in \eqref{eq-Q}.
   \REPEAT
   \STATE{1a. } Compute
$
\tx^{k+1} \in \mbox{argmin} \{ L(x,y_k) + \frac{1}{2}\norm{x-x_k}_S \}.
$
 \STATE{1b. }
Compute
$
\hx_{k+1} \in \mbox{argmin} \left\{ \begin{array}{l}
 \widehat{L}(x;x_k,y_k)
\end{array}
\right\}
$
\STATE{2. } Consider condition
$
\mbox{(I)} \quad L(\tx_{k+1},y_k)+\frac{1}{2}\norm{\tx_{k+1}-x_k}_S\leq
 L(\hx_{k+1},y_k) +\frac{1}{2}\norm{\hx_{k+1}-x_k}_S.
$
Set
\begin{eqnarray*}
x_{k+1} = \left\{ \begin{array}{ll}
 \tx_{k+1} & \mbox{if condition (I) holds}
\\[5pt]
\hx_{k+1} & \mbox{otherwise.}
\end{array}
\right.
\end{eqnarray*}
   \STATE{3. } $y^{k+1} = \argmin\{L(x^{k+1},y) + \frac{\beta_2}{2}\|y-y^k\|_T^2\}$
   \UNTIL{convergence}  \STATE {\bfseries Output:} Accumulation points $\overline{x}$ and $\overline{y}$
\end{algorithmic}
\end{algorithm}
In the above, $S$ and $T$ are given positive definite linear maps, and
$\norm{x-x^k}_S^2 = \inprod{x-x^k}{S(x-x^k)}$,
$\norm{y-y^k}_T^2 = \inprod{y-y^k}{T(y-y^k)}$.
Note that Step 1b is to safeguard against the possibility that the computed
$\tilde{x}_{k+1}$ is not a global optimal solution of the subproblem. We assume
that it is possible to compute the global optimal solution $\widehat{x}_{k+1}$
analytically.

Note that for our problem \eqref{eq:MC_Formulation_2},
the global minimizer of the nonconvex subproblem in Step 1b can be computed analytically as discussed in Section~\ref{sec:majorized_MC}. Next we show that any limit point of
$\{ (x_{k},y_{k})\}$ is a stationary point
of $L$ even if $\tx_{k+1}$ computed in Step 1a is not a global minimizer
of the subproblem.
\begin{theorem} \label{thm:critical_pt}
Let $\{(x_k,y_k)\}$ be the sequence generated by Algorithm~\ref{alg:PAM},
and $(\tx_{k+1},\hx_{k+1})$ are the intermediate iterates at iteration $k$.
\\[5pt]
(a) For all $k\geq 0$, we have that
\small
\begin{align}
&  L(\hx_{k+1},y_k) +  \frac{1}{2}\norm{\hx_{k+1}-x_k}_S^2
+\frac{1}{2}\norm{\hx_{k+1}-x_k}_{M-A^*A-S}^2\nonumber\\
&=\; \widehat{L}(\hx_{k+1};x_k,y_k)
 \;\leq\;  L(x_k,y_k),
\label{eq-thm1-1}
\\
& L(x_{k+1},y_k) + \frac{1}{2}\norm{x_{k+1}-x_k}_S^2  \leq  L(\hx_{k+1},y_k)+ \frac{1}{2}\norm{\hx_{k+1}-x_k}_S^2.
\label{eq-thm1-2}
\end{align}
\normalsize
(b) For all $k\geq 0$, we have that
\small
\begin{equation}
L(x_{k+1},y_{k+1}) + \frac{1}{2}\norm{x_{k+1}-x_k}_S^2 + \frac{1}{2}\norm{y_{k+1}-y_k}_T^2
  \leq   L(x_k,y_k).
\label{eq-thm1-3}
\end{equation}
\normalsize
Hence $\sum_{k=0}^\infty \norm{x_{k+1}-x_k}_S^2 +  \norm{y_{k+1}-y_k}_T^2 < \infty$
and $\displaystyle\lim_{k\rightarrow\infty} \norm{x_{k+1}-x_k} = 0 = \lim_{k\rightarrow\infty} \norm{y_{k+1}-y_k}$.
\\[5pt]
(c) Let $\{(x_{k'},y_{k'})\}$ be any convergent subsequence of $\{(x_k,y_k)\}$
with limit $(\bar{x},\bar{y})$. Then
$
 \lim_{k\rightarrow\infty} L(x_{k},y_{k}) =\lim_{k\rightarrow\infty}
 L(x_{k+1},y_{k}) =  \lim_{k\rightarrow\infty}
 \widehat{L}(\hx_{k+1};x_k,y_{k}) =  L(\bar{x},\bar{y}) .
$
Furthermore $\lim_{k\rightarrow\infty} \norm{\hx_{k+1}-x_k}=0$.
\\[5pt]
(d) Let $\{(x_{k'},y_{k'})\}$ be any convergent subsequence of $\{(x_k,y_k)\}$
with limit $(\bar{x},\bar{y})$. Then $(\bar{x},\bar{y})$ is a stationary point
of $L$.
\end{theorem}
The full proof is given in the Appendix. Here we explain the four parts of Theorem~\ref{thm:critical_pt}. Part(a) establishes the non-increasing monotonicity of the proximal regularized update. Leveraging on part(a), part(b) ensures the existence of the limits. Using Part(a), (b) and (c), (d) then shows the critical point convergence of Algorithm~\ref{alg:PAM}.

\subsection{Convex relaxation of \eqref{eq:MC_Formulation} as initialization} \label{sec:Convex_relaxation}

Due to the non-convexity of the rank and $\ell_0$ cardinality constraints, it is expected that the outcome of Algorithm~\ref{alg:PARSuMi} depends on initializations.
 A natural choice for the initialization of PARSuMi is the convex relaxation of both the rank
and $\ell_0$ function:
\begin{equation}\label{eq:RIRM_convex}
\min\Big\{
 f(W,E) + \lambda \norm{W}_* + \gamma \norm{E}_1 \mid
W\in \R^{m\times n}, E \in \R^{m\times n}_\Omega
\Big\}
\end{equation}
where $f(W,E) = \frac{1}{2}\norm{H\circ(W + E - \widehat{W}) }^2$, $\norm{\cdot }_*$ is the nuclear norm, and
 $\lambda$ and $\gamma$ are  regularization parameters.

Problem \eqref{eq:RIRM_convex} can be solved efficiently by the quadratic majorization-APG (accelerated proximal gradient) framework proposed by \citet{APG_NN}. At the $k$th iteration with iterate $(\bar{W}^k,\bar{E}^k)$, the majorization step replaces \eqref{eq:RIRM_convex} with a quadratic majorization of $f(W,E)$, so that $W$ and $E$ can be optimized independently, as we shall see shortly. Let $G^k = (H\circ H)\circ(\bar{W}^k+\bar{E}^k+\widehat{W})$.
By some simple algebra, we have
\begin{eqnarray*}
&& f(W,E) - f(\bar{W}^k,\bar{E}^k) = \frac{1}{2} \norm{H\circ(W-\bar{W}^k+E-\bar{E}^k)}^2
 \\
& &\qquad \qquad  +\; \inprod{W-\bar{W}^k+E-\bar{E}^k}{G^k}
\\
&& \leq \; \norm{W-\bar{W}^k}^2+\norm{E-\bar{E}^k}^2 +
 \inprod{W-\bar{W}^k+E-\bar{E}^k}{G^k}
\\[5pt]
&&= \norm{W-\widetilde{W}^k}^2 + \norm{E-\widetilde{E}^k}^2 +{\rm constant}
\end{eqnarray*}
where $\widetilde{W}^k = \bar{W}^k-G^k/2$ and
$\widetilde{E}^k = \bar{E}^k-G^k/2$.
At each step of the APG method, one minimizes \eqref{eq:RIRM_convex} with $f(W,E)$
replaced by the above quadratic majorization. As the resulting problem is separable in
$W$ and $E$, we can minimize them separately, thus
yielding the following two optimization problems:
\begin{align}
W^{k+1} = \mbox{argmin} \; &\frac{1}{2} \norm{ W - \widetilde{W}^k}^2 +
\frac{\lambda}{2} \norm{W}_* \label{eq:RIRM_convex_majorized_W}\\
E^{k+1} = \mbox{argmin} \; & \frac{1}{2}\norm{ E - \widetilde{E}^k}^2 +
\frac{\gamma}{2} \norm{ E}_1 \label{eq:RIRM_convex_majorized_E}
\end{align}
The main reason for performing the above majorization
 is because the solutions to \eqref{eq:RIRM_convex_majorized_W}
and \eqref{eq:RIRM_convex_majorized_E} can
readily be found with closed-form solutions. For \eqref{eq:RIRM_convex_majorized_W}, the minimizer is given by the Singular Value Thresholding (SVT) operator. For \eqref{eq:RIRM_convex_majorized_E}, the minimizer is given by the well-known soft thresholding operator \citep{donoho1995softthresh}.
The APG algorithm, which is adapted from
\citet{APG_L1} and analogous to that in \citet{APG_NN}, is summarized below.

\begin{algorithm}[htb]
   \caption{An APG algorithm for \eqref{eq:RIRM_convex}}
   \label{alg:APG}
\begin{algorithmic}
   \STATE {\bfseries Input:} Initialize $W^0=\bar{W}^0=0$, $E^0=\bar{E}^0=0$,
$t_0=1$, $k=0$
   \REPEAT
   \STATE{1. } Compute $G^k = (H\circ H)\circ(\bar{W}^k+\bar{E}^k+\widehat{W})$,
$\widetilde{W}^k$, $\widetilde{E}^k$.
   \STATE{2. } Update $W^{k+1}$ by applying the SVT on $\widetilde{W}^k$
in  \eqref{eq:RIRM_convex_majorized_W}.
   \STATE{3. } Update $E^{k+1}$ by applying the soft-thresholding operator on $\widetilde{E}^k$ in \eqref{eq:RIRM_convex_majorized_E}.
\STATE{4.} Update step size $t_{k+1} = \frac{1}{2}(1+\sqrt{1+4t_k^2})$.
 \STATE{5.} $(\bar{W}^{k+1},\bar{E}^{k+1}) = (W^{k+1},E^{k+1}) + \frac{t_k-1}{t_{k+1}}
  (W^{k+1}-W^k,E^{k+1}-E^k) $
   \UNTIL{Convergence}
   \STATE {\bfseries Output:} Accumulation points $\overline{W}$ and $\overline{E}$
\end{algorithmic}
\end{algorithm}
As has already been proved in \citet{APG_L1},
the APG algorithm, including the one above, has a very nice worst case iteration complexity result in that for any given $\epsilon>0$, the APG algorithm needs at most
$O(1/\sqrt{\epsilon})$ iterations to compute an $\epsilon$-optimal (in terms of
function value) solution.

The tuning of the regularization parameters $\lambda$ and $\gamma$
in \eqref{eq:RIRM_convex} is fairly straightforward. For $\lambda$, we use the singular values of the converged $\overline{W}$ as a reference. Starting from a
relatively large value of $\lambda$, we reduce it by a constant factor in each pass to obtain a
$\overline{W}$ such that its singular values beyond the $r$th are much smaller than
the first $r$ singular values. For $\gamma$, we use the suggested value of $1/\sqrt{\max(m,n)}$ from RPCA \citep{Candes2011_JACM}. In our experiments, we find that we only need a ballpark figure, without having to do a lot of tuning. Taking $\lambda=0.2$ and $\gamma=1/\sqrt{\max(m,n)}$ serve the purpose well.

\subsection{Other heuristics} \label{sec:heuristics}
In practice, we design two heuristics to further boost the quality of the convex initialization. These are tricks that allow PARSuMi to detect corrupted entries better and are always recommended.

We refer to the first heuristic as ``Huber Regression''. The idea is that the quadratic loss term in our matrix completion step \eqref{eq:MC_W} is likely to result in a dense spread of estimation error across all measurements. There is no guarantee that those true corrupted measurements will hold larger errors comparing to the uncorrupted measurements. On the other hand, we note that the quality of the subspace $N^k$ obtained from LM\_GN is usually good despite noisy/corrupted measurements. This is especially true when the first LM\_GN step is initialized with Algorithm~\ref{alg:APG}. Intuitively, we should be better off with an intermediate step, using $N^{k+1}$ to detect the errors instead of  $W^{k+1}$, that is, keeping $N^{k+1}$ as a fixed input and finding coefficient $C$ and $E$ simultaneously with
\begin{equation}\label{eq:use_N_E}
\begin{aligned}
& \min_{E,C}
\; \frac{1}{2}\norm{H\circ (N^{k+1}C-\widehat{W}+E)}^2\\
& \text{subject to}
\quad \|E\|_0\leq N_0.
\end{aligned}
\end{equation}
To make it computationally tractable, we relax \eqref{eq:use_N_E} to
\begin{equation}\label{eq:use_N_E_convex}
\begin{aligned}
& \min_{E,C}
\; \frac{1}{2}\norm{H\circ (N^{k+1}C-\widehat{W}+E)}^2+\eta_0\|E\|_1\\
\end{aligned}
\end{equation}
where $\eta_0 >0$ is a penalty parameter. Note that each column of the above problem can be decomposed into the following Huber loss regression problem  ($E$ is absorbed into the Huber penalty)
\begin{equation}\label{eq:use_N_E_huber}
\begin{aligned}
& \min_{C_j} \sum_{i=1}^m
 \mbox{Huber}_{\eta_0/H_{ij}}(H_{ij} ((N^{k+1} C_j)_i-\widehat{W}_{ij})). \\
\end{aligned}
\end{equation}
Since $N^{k+1}$ is known, \eqref{eq:use_N_E_convex} can be solved very efficiently using
the APG algorithm, whose derivation is similar to that of Algorithm~ \ref{alg:APG}, with soft-thresholding operations on $C$ and $E$. To further reduce the Robin Hood effect (that haunts all $\ell_1$-like penalties) and enhance sparsity, we may optionally apply the iterative re-weighted Huber minimization (a slight variation of the method in \citet{candes2008enhancing}), that is, solving \eqref{eq:use_N_E_huber} for $l_{max}$ iterations using an entrywise weighting factor inversely proportional to the previous iteration's fitting residual. In the end, the optimal columns $C_j$'s  are concatenated into the optimal solution matrix $C^*$ of \eqref{eq:use_N_E_convex}, and we set
$$W^{k+1}=N^{k+1}C^*.$$
With this intermediate step between the $W$ step and the $E$ step, it is much easier for the $E$ step to detect the support of the actual corrupted entries.

The above procedure can be used in conjunction with another heuristic that avoids adding false positives into the corruption set in the $E$ step when the subspace $N$ has not yet been accurately recovered. This is achieved by imposing a threshold $\eta$ on the minimum absolute value of $E^k$'s non-zero entries, and shrink this threshold by a factor (say 0.8) in each iteration. The ``Huber regression'' heuristic is used only when $\eta>\eta_0$, and hence only in a very small number of iteration before the support of $E$ has been reliably recovered. Afterwards the pure PARSuMi iterations (without the Huber step) will take over, correct the Robin Hood effect of Huber loss and then converge to a high quality solution.

Note that our critical point convergence guarantee in Section~\ref{sec:convergence} is not hampered at all by the two heuristics, since after a small number of iterations, $\eta\leq\eta_0$ and we come back to the pure PARSuMi.

\section{Experiments and discussions}\label{sec:experiment}

In this section, we present the methodology and results of various experiments designed to evaluate the effectiveness of our proposed method. The experiments revolve around synthetic data and two real-life datasets: the Oxford Dinosaur sequence, which is representative of data matrices in SfM works, and the Extended YaleB face dataset \citep{lee2005extendedyaleb}, which we use to demonstrate how PARSuMi works on photometric stereo problems.

In the synthetic data experiments, our method is compared with the state-of-the-art algorithms for the objective function in \eqref{eq:MC_L1} namely Wiberg $\ell_1$ \citep{WibergL1} and GRASTA \citep{he2011grasta}. ALP and AQP \citep{ke2005L1} are left out since they are shown to be inferior to Wiberg $\ell_1$ in \citet{WibergL1}. For the sake of comparison, we perform the experiment on recovery effectiveness using the same small matrices as in Section 5.1 of \citet{WibergL1}. Other synthetic experiments are conducted with more reasonably-sized matrices. Whenever appropriate, we also include a comparison to a variant of RPCA that handles missing data \citep{Wu_photometric} which solves \eqref{eq:RMC_nuc} using the augmented Lagrange multiplier (ALM) algorithm (we will call it ALM-RPCA from here onwards). This serves as a representative of the nuclear norm based methods.



The real data from the SfM and photometric stereo problems contain many challenges typical in practical scenarios. They contain large contiguous areas of missing data, and potentially highly corrupted observations which may not be sparse too. For instance, in the YaleB face dataset, grazing illumination tends to produce large area of missing data (well over 50\%) and often large number of outliers too (due to specular highlights). The PARSuMi method outperformed a variety of other methods in the experiments, even uncovering hitherto unknown corruptions inherent in the Dinosaur data from SfM. The results also corroborate those obtained in the synthetic data experiments, in that our method can handle a substantially larger fraction of missing data and corruptions, thus providing empirical evidence for the efficacy of PARSuMi under practical scenarios.

For a summary of the parameters used in the experiments, please refer to the Appendix. 





\subsection{Convex Relaxation as an Initialization Scheme}
We first investigate the results of our convex initialization scheme by testing on a randomly generated $100\times 100$ rank-4  matrix. A random selection of 70\% and 10\% of the entries are considered missing and corrupted respectively. Corruptions are generated by adding large uniform noise between $[-1,1]$. In addition, Gaussian noise $\mathcal{N}(0,\sigma)$ for $\sigma=0.01$ is added to all observed entries. From Fig.~\ref{fig:E_robinhood}, we see that the convex relaxation outlined in Section~\ref{sec:Convex_relaxation} was able to recover the error support, but there is considerable difference in magnitude between the recovered error and the ground truth, owing to the ``Robin Hood'' attribute of $\ell_1$-norm as a convex proxy of $\ell_0$. Nuclear norm as a proxy of rank also suffers from the same woe. Similar observations can be made on the results of the Dinosaur experiments, which we will show later.



\begin{figure}[bt]
  \centering
  \includegraphics[width=0.49\linewidth]{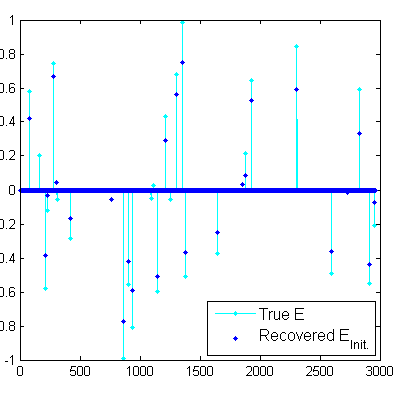}
  \includegraphics[width=0.49\linewidth]{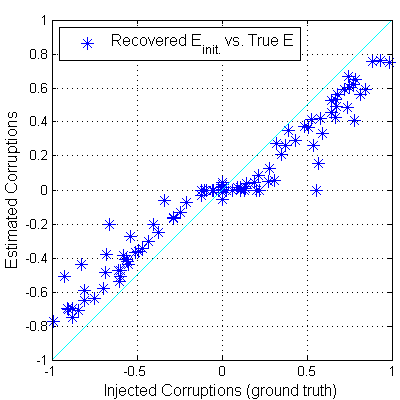}\\
  \caption{The Robin Hood effect of Algorithm~\ref{alg:APG} on detected sparse corruptions $E_{\mathrm{Init}}$. \textbf{Left}: illustration of a random selection of detected E vs. true E. Note that the support is mostly detected, but the magnitude falls short. \textbf{Right}: scatter plot of the detected E against true E (perfect recovery falls on the $y=x$ line, false positives on the $y$-axis and false negatives on the $x$-axis).}\label{fig:E_robinhood}
\end{figure}

Despite the problems with the solution of the convex initialization, we find that it is a crucial step for PARSuMi to work well in practice. As can be seen from Fig.~\ref{fig:E_robinhood}, the detected error support can be quite accurate. This makes the $E$-step of PARSuMi more likely to identify the true locations of corrupted entries.


\subsection{Impacts of poor initialization}
When the convex initialization scheme fails to obtain the correct support of the error, the ``Huber Regression'' heuristic may help PARSuMi to identify the support of the corrupted entries.  We illustrate the impact by intentionally mis-tuning the parameters of Algorithm~\ref{alg:APG} such that the initial $E$ bears little resemblance to the true injected corruptions. Specifically, we test the cases when the initialization fails to detect many of the corrupted entries (false negatives) and when many entries are wrongly detected as corruptions (false positives). From Fig.~\ref{fig:poor_initialization}, we see that PARSuMi is able to recover the corrupted entries to a level comparable to the magnitude of the injected Gaussian noise in both experiments\footnote{Note that a number of false positives persist in the second experiment. This is understandable because false positives often contaminate an entire column or row, making it impossible to recover that column/row in later iterations even if the subspace is correctly detected. To avoid such an undesirable situation, we prefer ``false negatives'' over ``false positives'' when tuning Algorithm~\ref{alg:APG}. In practice, it suffices to keep the initial $E$ relatively sparse.}.
\begin{figure}
  \centering
  \subfigure[\small{False negatives}]{
  \includegraphics[width=0.46\linewidth]{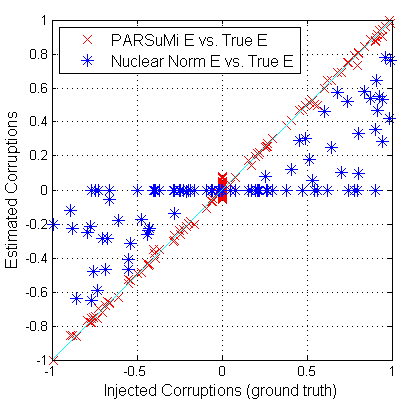}
  }
  \subfigure[\small{False positives}]{
  \includegraphics[width=0.46\linewidth]{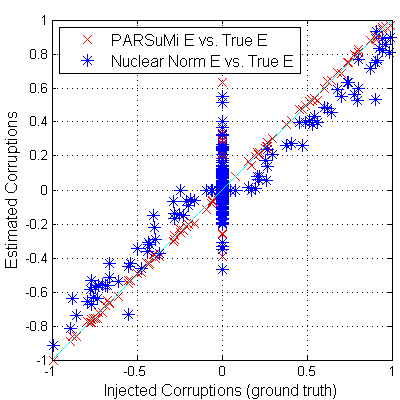}
  }\\
  \caption{Recovery of corruptions from poor initialization.}\label{fig:poor_initialization}
\end{figure}

In most of our experiments, we find that PARSuMi is often able to detect the corruptions perfectly from a simple initializations with all zeros, even without the ``Huber Regression'' heuristic. This is especially true when the data are randomly generated with benign sampling pattern and well-conditioned singular values. However, in challenging applications such as SfM, a good convex initialization and the ``Huber Regression'' heuristic are always recommended.

\subsection{Recovery effectiveness from sparse corruptions} \label{sec:Recovery_Effectiveness}

For easy benchmarking, we use the same synthetic data in Section~5.1 of \citet{WibergL1} to investigate the quantitative effectiveness of our proposed method. A total of 100 random low-rank matrices with missing data and corruptions
are generated and tested using PARSuMi, Wiberg~$\ell_1$ and GRASTA.

In accordance with \citet{WibergL1}, the ground truth low rank matrix $W \in \mathbb{R}^{m \times n},m=7,n=12,r=3,$ is generated as $W = UV^T$, where $U\in \mathbb{R}^{m \times r}, V\in \mathbb{R}^{n \times r}$ are generated using uniform distribution, in the range [-1,1]. 20\% of the data are designated as missing, and 10\% are added with corruptions, both at random locations. The magnitude of the corruptions follows a uniform distribution $[-5,5]$.
Root mean square error (RMSE) is used to evaluate the recovery precision:
\begin{equation} \label{eq:RMS}
\mathrm{RMSE} := \frac{\Vert W_{\mathrm{recovered}} - W \Vert_F}{\sqrt{mn}}.
\end{equation}
Out of the 100 independent experiments, the number of runs that returned RMSE values of
less than 5 are 100 for PARSuMi, 78 and 58 for  Wiberg~$\ell_1$ (with two different initializations) and similarly 94 and 93 for GRASTA. These are summarized in Fig.~\ref{fig:Synthetic_hist}.
\begin{figure}[tb]
 \centering
 \includegraphics[width=8cm,height=4.5cm]{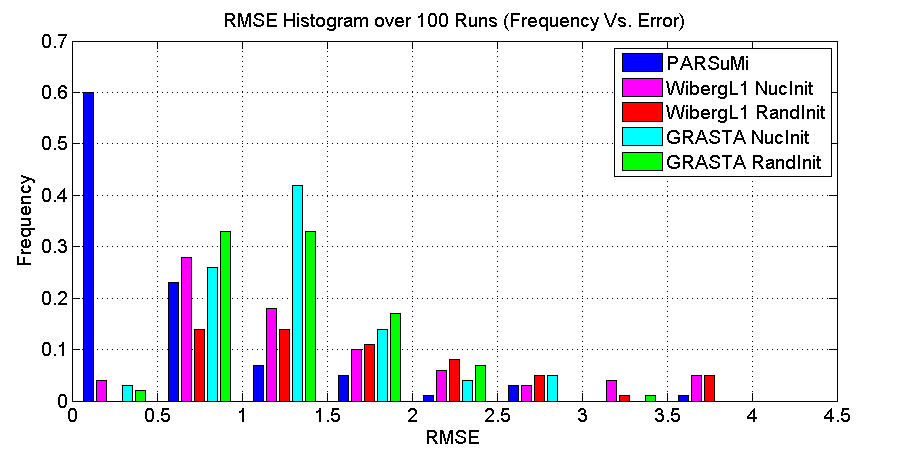}
 \caption{\small{A histogram representing the frequency of different magnitudes of RMSE in the estimates generated by each method. }}
 \label{fig:Synthetic_hist}
\end{figure}

\subsection{Recovery under varying level of corruptions, missing data and noise}
To gain a holistic understanding of our proposed method, we perform a series of systematically parameterized experiments on $40\times60$ rank-4 matrices
(with the elements of the factors $U,V$ drawn independently from
the uniform distribution on $[-1,1]$), with conditions ranging from 0-80\% missing data, 0-20\% corruptions of range [-2,2], and Gaussian noise with $\sigma$ in the range [0,0.1]. By fixing the Gaussian noise at a specific level, the results are rendered in terms of phase diagrams showing the recovery precision as a function of the missing data and outliers. The precision is quantified as the difference between the recovered RMSE and the oracle bound RMSE \footnote{See the Appendix for details.}. As can be seen from Fig.~\ref{fig:phase_diagram}, our algorithm obtains near optimal performance at an impressively large range of missing data and outlier at $\sigma=0.01$\footnote{The phase diagrams for other levels of noise look very much like Fig.~\ref{fig:phase_diagram}; we therefore did not include them in the paper.}.

For comparison, we also displayed the results for closely related methods, e.g., ALM-RPCA~\citep{Wu_photometric}, GRASTA~\citep{he2011grasta}, DRMF~\citep{DRMF}, LM\_GN~\citep{chen2011hessian} as well as Algorithm~\ref{alg:APG} (our initialization). Wiberg~$\ell_1$ is omitted because it is too slow. Among all the methods we compared, PARSuMi is able to successfully reconstruct the largest range of matrices with almost optimal numerical accuracy. Also, the results for DRMF and LM\_GN are well-expected since they are not designed to handle both missing data and outliers.


\begin{figure}[tb]
 \centering
\subfigure[\small{PARSuMi}]{
 \includegraphics[width=0.465\linewidth]{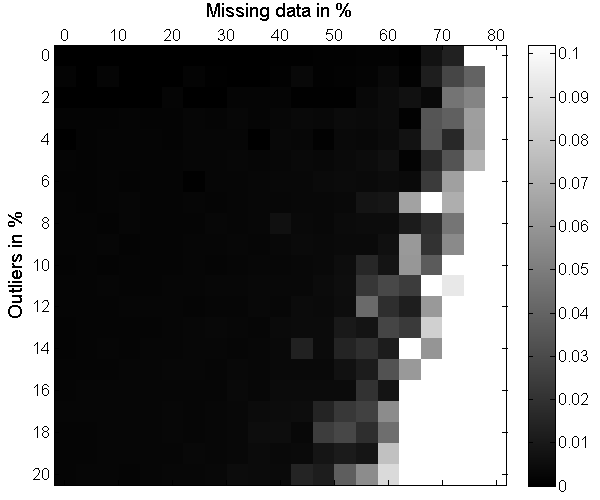}
 \label{fig:Phase_diag_PARSuMi}
}
\subfigure[\small{Our initialization}]{
\includegraphics[width=0.465\linewidth]{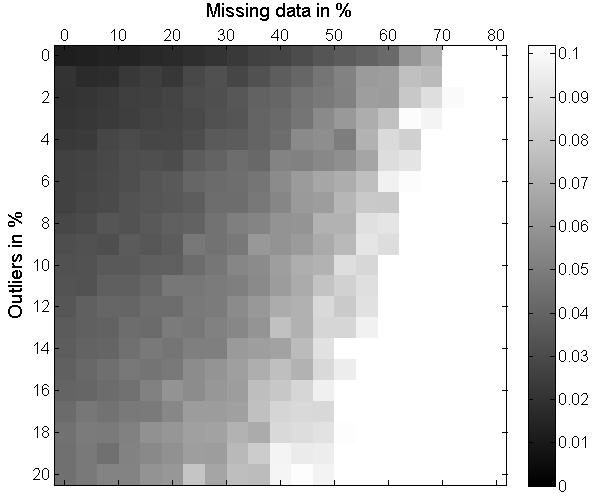}
\label{fig:Phase_diag_APG}
}
\subfigure[\small{ALM-RPCA}]{
\includegraphics[width=0.465\linewidth]{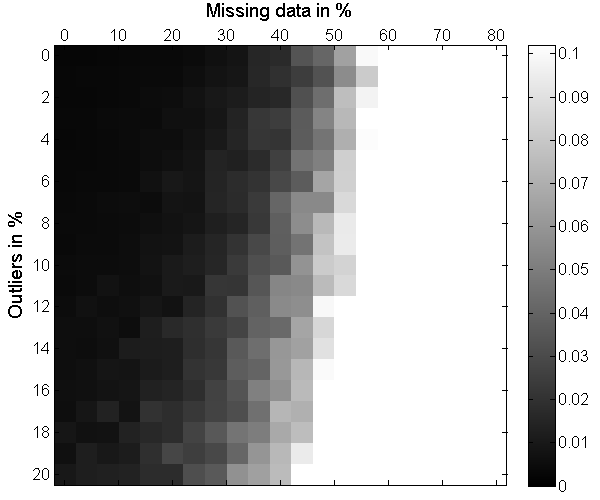}
\label{fig:Phase_diag_RPCA}
}
\subfigure[\small{GRASTA}]{
\includegraphics[width=0.465\linewidth]{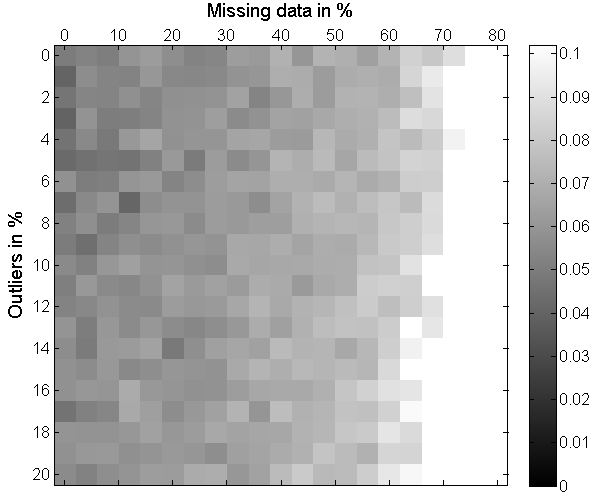}
\label{fig:Phase_diag_GRASTA}
}
\subfigure[\small{DRMF}]{
\includegraphics[width=0.465\linewidth]{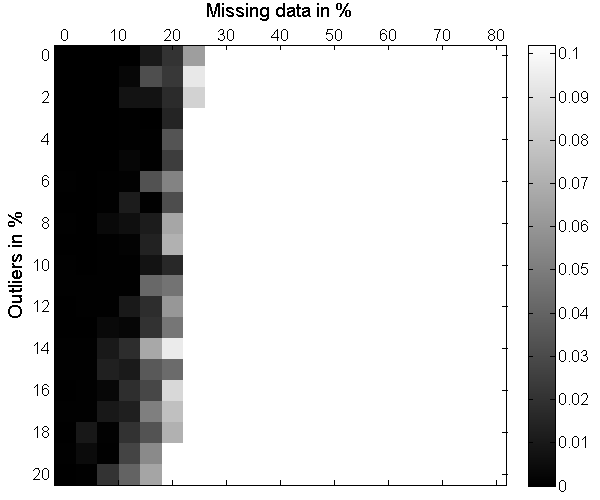}
\label{fig:Phase_diag_DRMF}
}
\subfigure[\small{LM\_GN}]{
\includegraphics[width=0.465\linewidth]{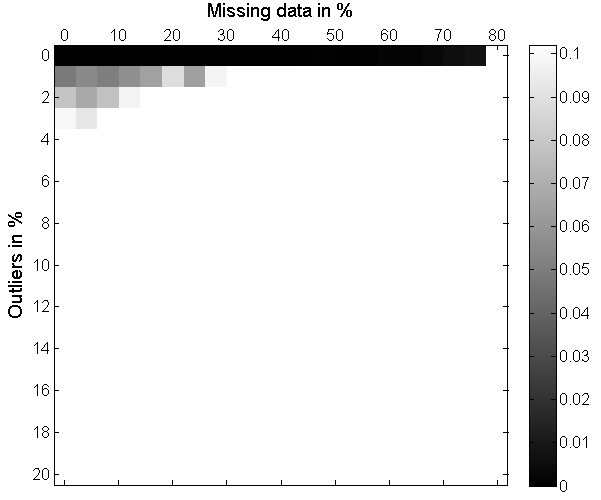}
\label{fig:Phase_diag_LMGN}
}
 \caption[Optional caption for list of figures]{\small{Phase diagrams (darker is better) of RMSE with varying proportion of missing data and corruptions with Gaussian noise $\sigma=0.01$.}}
 \label{fig:phase_diagram}
\end{figure}



\subsection{SfM with missing and corrupted data on Dinosaur} \label{sec:expt_outlier}

In this section, we apply PARSuMi to the problem of SfM using the Dinosaur sequence and investigate how well the corrupted entries can be detected and recovered in real data. We have normalized image pixel dimensions (width and height) to be in the range [0,1]; all plots, unless otherwise noted, are shown in the normalized coordinates.

To simulate 
data corruptions arising from wrong feature matches,
we randomly add sparse error of the range [-2,2]\footnote{In SfM data corruptions are typically matching failures. Depending on where true matches are, error induced by a matching failure can be arbitrarily large. If we constrain true match to be inside image frame $[0,1]$(which is often not the case), then the maximum error magnitude is $1$. We found it appropriate to at least double the size to account for general matching failures in SfM, hence $[-2,2]$.} to 1\% of the sampled entries. This is a more realistic (and much larger\footnote{[-50,50] in pixel is only about [-0.1,0.1] in our normalized data, which could hardly be regarded as ``gross'' corruptions.}) definition of outliers for SfM compared to the [-50,50] pixel range used to evaluate Wiberg $\ell_1$ in \citet{WibergL1}.
\begin{table}[t]
\centering
\scriptsize{
\begin{tabular}{|p{2.4cm}|c|c|c|}
  \hline
   & PARSuMi & Wiberg $\ell_1$ & GRASTA\\
   \hline
  No. of success & 9/10 & 0/10& 0/10\\
  \hline
  Run time (mins):
  min/avg/max  & 2.2/2.9/5.2 &  76/105/143 & 0.2/0.5/0.6 \\
  \hline
  Min RMSE (original pixel unit)& 1.454  & 2.715& 22.9\\
  \hline
  Min RMSE excluding corrupted entries & 0.3694 & 1.6347 &21.73\\
  \hline
\end{tabular}
}
\caption{Summary of the Dinosaur experiments. Note that because there is no ground truth for the missing data, the RMSE is computed only for those observed entries as in \citet{Damped_Newton_2005}. }
\label{tab:SummaryDinosaur}
\end{table}

We conducted the experiment 10 times each for PARSuMi, Wiberg $\ell_1$ (with SVD initialization) and GRASTA (random initialization as recommended in the original paper) and count the number of times they succeed. As there are no ground truth to compare against, we cannot use the RMSE to evaluate the quality of the filled-in entries. Instead, we plot the feature trajectory of the recovered data matrix for a qualitative judgement. As is noted in \citet{Damped_Newton_2005}, a correct recovery should consist of all elliptical trajectories. Therefore, if the recovered trajectories look like that in Fig.~\ref{fig:Traj_DN_global}, we count the recovery as a success.
\begin{figure}[ht]
 \centering
  \subfigure[\scriptsize{Input}]{
 \includegraphics[width=0.465\linewidth]{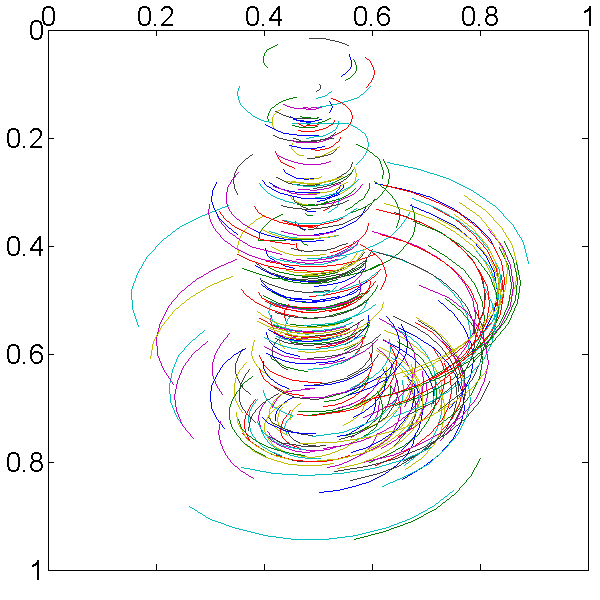}
 \label{fig:Traj_Input}
 }
  \subfigure[\scriptsize{ALM-RPCA}]{
 \includegraphics[width=0.465\linewidth]{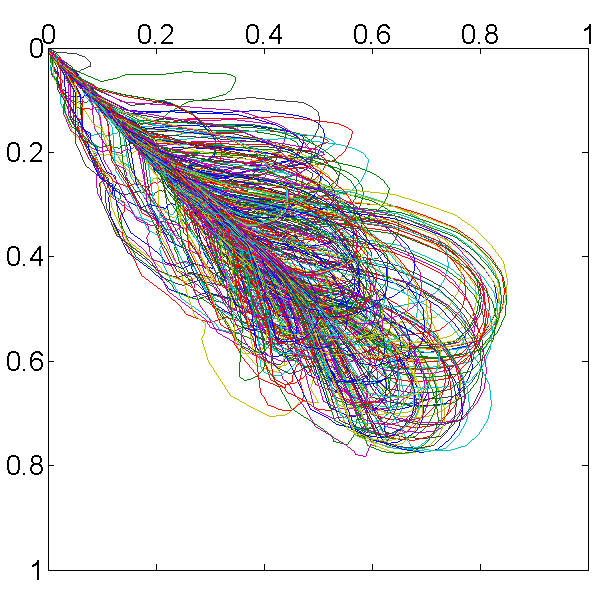}
 \label{fig:Traj_ALM}
 }

 \subfigure[\scriptsize{Wiberg $\ell_1$}]{
 \includegraphics[width=0.465\linewidth]{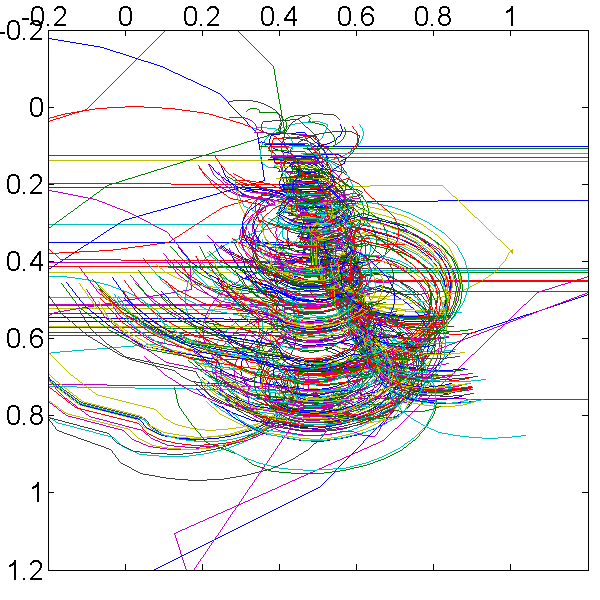}
 \label{fig:Traj_Wibergl1}
 }
  \subfigure[\scriptsize{GRASTA}]{
 \includegraphics[width=0.465\linewidth]{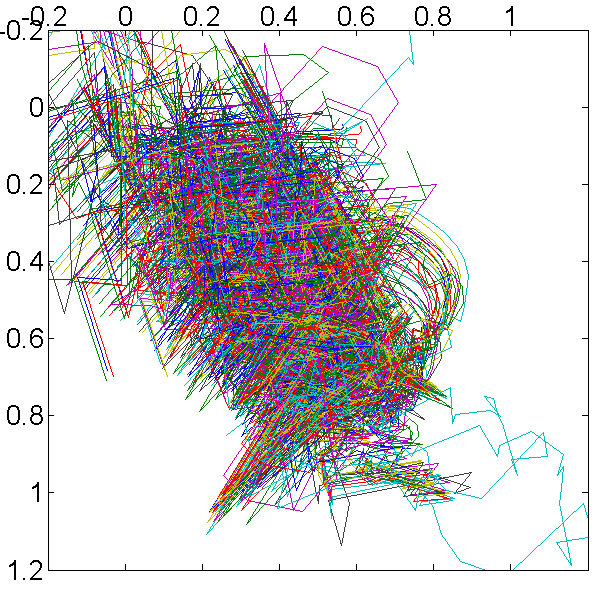}
 \label{fig:Traj_GRASTA}
 }
  \subfigure[\scriptsize{Our initialization}]{
 \includegraphics[width=0.465\linewidth]{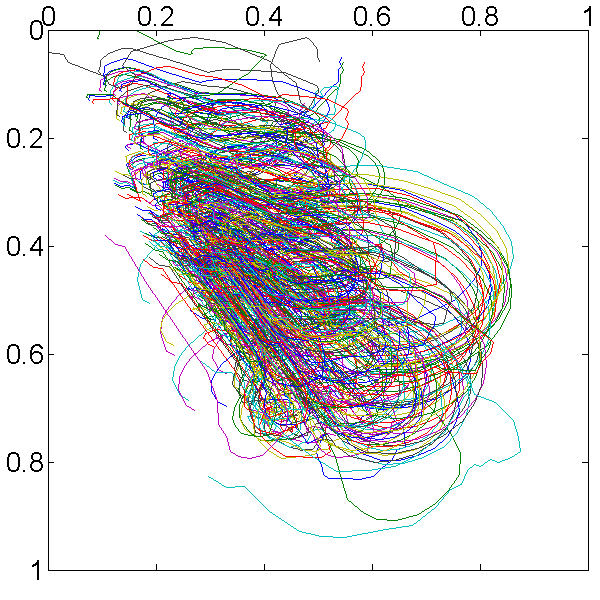}
 \label{fig:Traj_Convex}
 }
 \subfigure[\scriptsize{PARSuMi}]{
 \includegraphics[width=0.465\linewidth]{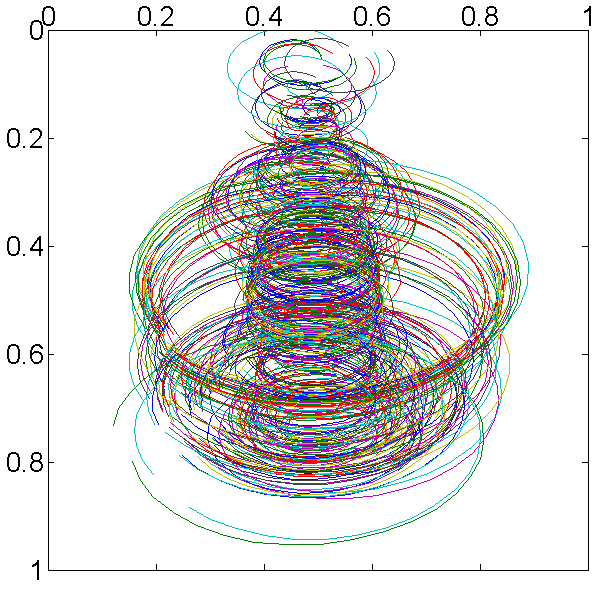}
 \label{fig:Traj_PARSuMi}
 }
 \caption{\small{Comparison of recovered feature trajectories with different methods. It is clear that under dense noise and gross outliers, neither convex relaxation nor $\ell_1$ error measure yields satisfactory results. Solving the original non-convex problem with \subref{fig:Traj_Convex} as an initialization produces a good solution. }}
 \label{fig:trajectory comparison}
\end{figure}

The results are summarized in Table~\ref{tab:SummaryDinosaur}. Notably, PARSuMi managed to correctly detect the corrupted entries and fill in the missing data in 9 runs while Wiberg~$\ell_1$ and GRASTA failed on all 10 attempts. Typical feature trajectories recovered by each method are shown in Fig.~\ref{fig:trajectory comparison}. Note that only PARSuMi is able to recover the elliptical trajectories satisfactorily.

For comparison, we also include the input (partially observed trajectories) and the results of our convex initialization in Fig.~\ref{fig:Traj_Input}~and~\ref{fig:Traj_Convex} respectively.

An interesting and somewhat surprising finding is that the result of PARSuMi is even better than the global optimal solution for data containing supposedly no corruptions (and thus can be obtained with $\ell_2$ method) (see Fig.~\ref{fig:Traj_DN_global}, which is obtained under no corruptions in the observed data)! In particular, the trajectories are now closed.

The reason becomes clear when we look at Fig.~\ref{fig:Convex_Recovery_Residue}, which shows two large spikes in the vectorized difference between the artificially injected corruptions and the recovered corruptions by PARSuMi. This suggests that there are hitherto unknown corruptions inherent in the Dinosaur data. We trace the two large ones into the raw images, and find that they are indeed data corruptions corresponding to mismatched feature points from the original dataset; our method managed to recover the correct feature matches (left column of Fig.~\ref{fig:Outliers_in_the_Dinosaur_data}).


\begin{figure}[tb]
 \centering
 \subfigure[\scriptsize{Initialization via Algorithm~\ref{alg:APG} and the final recovered errors by PARSuMi (Algorithm~\ref{alg:PARSuMi})}]{
 \includegraphics[width=\linewidth]{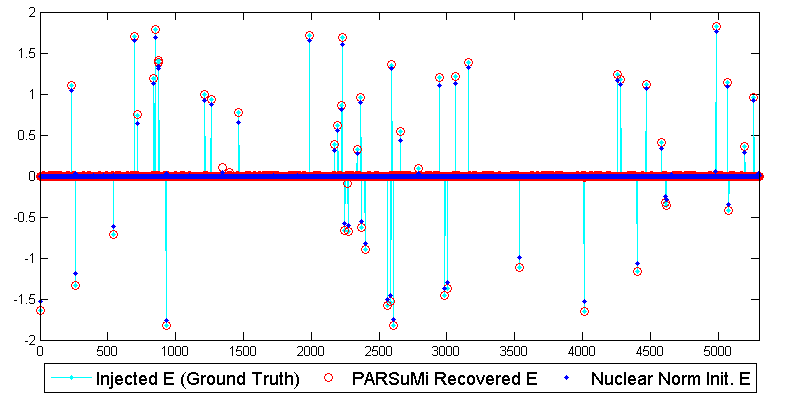}
 \label{fig:Convex_Recovery_Groundtruth}
 }
 \subfigure[\scriptsize{Difference of the recovered and ground truth error (in original pixel unit)}]{
 \includegraphics[width=\linewidth]{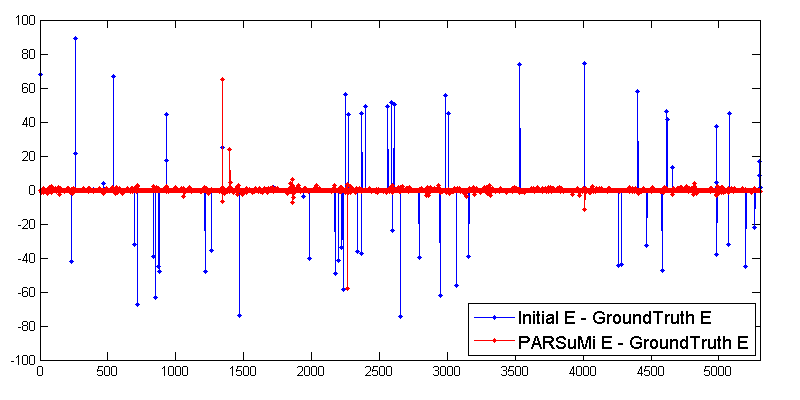}
 \label{fig:Convex_Recovery_Residue}
 }
 \caption[Optional caption for list of figures]{\small{Sparse corruption recovery in the Dinosaur experiments: The support of all injected outliers are detected  by Algorithm~\ref{alg:APG} (see \subref{fig:Convex_Recovery_Groundtruth}), but the magnitudes fall short by roughly 20\% (see \subref{fig:Convex_Recovery_Residue}). Algorithm~\ref{alg:PARSuMi} is able to recover all injected sparse errors, together with the inherent tracking errors in the dataset (see the red spikes in \subref{fig:Convex_Recovery_Residue}).}}
 \label{fig:Convex_error_recovery}
\end{figure}
\begin{figure}[tb]
\centering
\subfigure[Zoom-in view: recovered matching error in frame 13]{
\includegraphics[width=8cm]{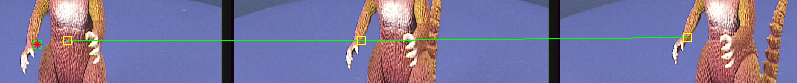}
\label{fig:Outlier_in_frame_13Zoom}
}
 \subfigure[Zoom-in view: recovered matching error in frame 15]{
 \includegraphics[width=8cm]{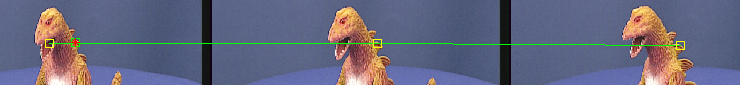}
 \label{fig:Outlier_in_frame_15Zoom}
 }
\caption{\small{Original tracking errors in the Dinosaur data identified (yellow box) and corrected by PARSuMi (green box with red star) in frame 13 feature 86
 \subref{fig:Outlier_in_frame_13Zoom} and frame 15 feature 144
 \subref{fig:Outlier_in_frame_15Zoom}.}}
\label{fig:Outliers_in_the_Dinosaur_data}
\end{figure}

The result shows that PARSuMi recovered not only the artificially added errors, but also the intrinsic errors in the data set. In \citet{Damped_Newton_2005}, it was observed that there is a mysterious increase of the objective function value upon closing the trajectories by imposing orthogonality constraint on the factorized camera matrix. Our discovery of these intrinsic tracking errors explained this matter evidently. It is also the reason why the $\ell_2$-based algorithms (see Fig.~\ref{fig:Traj_DN_global}) find a global minimum solution that is of poorer quality (trajectories fail to close loop).

To complete the story, we generated the 3D point cloud of Dinosaur with the completed data matrix. The results viewed from different directions are shown in Fig.~\ref{fig:Point cloud}.
\begin{figure}[ht]
 \centering
 \includegraphics[width=0.47\linewidth]{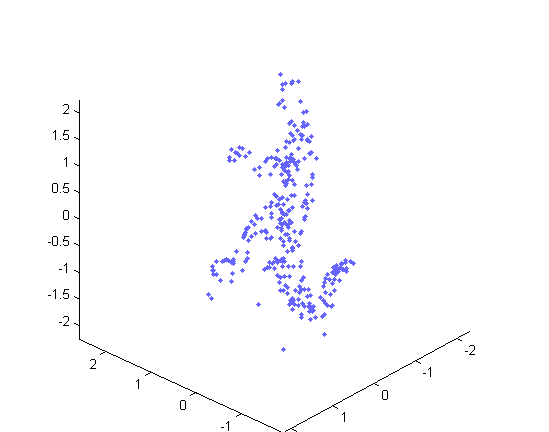}
 \includegraphics[width=0.47\linewidth]{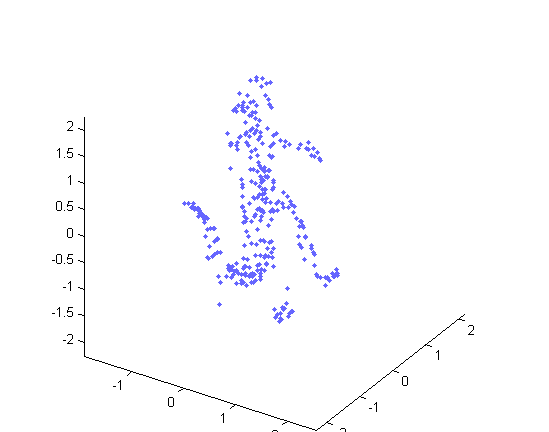}
 \caption[Optional caption for list of figures]{3D point cloud of the reconstructed Dinosaur.}
 \label{fig:Point cloud}
\end{figure}

\subsection{Photometric Stereo}
Another intuitive application for PARSuMi is photometric stereo, a problem of reconstructing the 3D shape of an object from images taken under different lighting conditions. In the most ideal case of Lambertian surface model (diffused reflection), the data matrix obtained by concatenating vectorized images together is of rank 3.

Real surfaces are of course never truly Lambertian. There are usually some localized specular regions appearing as highlights in the image.  Moreover, since there is no way to obtain a negative pixel value, all negative inner products will be observed as zero. This is the so-called attached shadow. Images of non-convex object often also contain cast shadow, due to the blocking of light path. If these issues are teased out, then the seemingly naive Lambertian model is able to approximate many surfaces very well.

\citet{Wu_photometric} subscribed to this low-rank factorization model and proposed to model all dark regions as missing data, all highlights as sparse corruptions and then use a variant of RPCA (identical to \eqref{eq:RMC_nuc}) to recover the full low-rank matrix. The solution however is only tested on noise-free synthetic data and toy-scale real examples.
\citet{DelBue2012balm} applied their BALM on photometric stereo too, attempting on both synthetic and real data. Their contribution is to impose the normal constraint of each normal vector during the optimization.

We compare PARSuMi with the aforementioned two methods on several photometric stereo datasets. Quantitatively, we use the Caesar and Elephant data in \citet{Wu_photometric} and compare the reconstructed surface normal against the ground truth. The data is illustrated in Fig.~\ref{fig:CaesarElephant} and the comparison is detailed in Table~\ref{tab:photometrc}. As we can see, PARSuMi has the smallest reconstruction error among the three methods in all experiments.
\begin{figure}
  \centering
  \includegraphics[width=\linewidth]{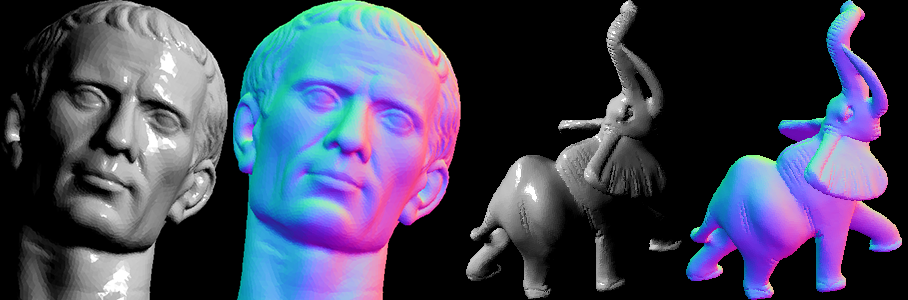}\\
  \caption{Illustration of the synthetic data and their surface normal. Note that there are specular regions and shadows.}\label{fig:CaesarElephant}
\end{figure}

\begin{table*}
  \centering
  \begin{tabular}{|l|l|p{3cm}|p{3cm}|p{3cm}|}
     \hline
     Dataset & PARSuMi & ALM-RPCA \citep{Wu_photometric} & BALM \citep{DelBue2012balm} & Oracle (a lower bound) \\
      \hline
     Elephant & \textbf{7.13e-2} (16.7 min) & 7.87e-2 (1.1 min) & 3.55 (1.1 min) & model error\\
     Caesar & \textbf{1.83e-1} (28.6 min) & 2.71e-1 (7.2 min) & 3.11 (5.2 min) & model error\\
     Elephant + $\mathcal{N}(0,0.05)$ & \textbf{2.35} (28.3 min) & 2.62 (1.5 min) & 4.37 (1.1 min)& 1.70 + model error\\
     Caesar + $\mathcal{N}(0,0.05)$ & \textbf{2.34 }(99.2 min) & 2.53 (8.3 min)& 4.06  (6.6 min) &1.73 + model error\\
     \hline
   \end{tabular}
  \caption{Angular error (in degree) and runtime (in minutes) comparison for the synthetic data photometric stereo experiments. The lowest estimation error is labeled in boldface. The oracle column gives the information-theoretic limit, it depends on an unknown model error as we are using a Lambertian model to deal with the data rendered by the Cook-Torrence model.}\label{tab:photometrc}
\end{table*}

We also conducted a qualitative comparison of the methods on a real-life data using Subject 3 in the Extended YaleB dataset since it was initially used to evaluate BALM in \citet{DelBue2012balm}\footnote{The authors claimed that it is Subject 10 \citep[Figure~9]{DelBue2012balm}, but careful examination of all faces shows that it is in fact Subject 3.}. As we do not have any ground truth, we can only compare the reconstruction qualitatively.

From Fig.~\ref{fig:YaleB_comparison}, we can clearly see that PARSuMi is able to recover the missing pixels in the image much better than the other two methods. In particular, Fig.~\ref{subfig:missing_data_comparison} and \ref{subfig:missing_data_comparison_zoom-in} shows that PARSuMi's reconstruction (in the illuminated half of the face) has fewest artifacts. This can be seen from the unnatural grooves that the red arrows point to in Fig.~\ref{subfig:missing_data_comparison_zoom-in}. Moreover, we know from the original image that the light comes from the right-hand-side of the subject; thus all the pixels on the left side of his face (e.g. the red ellipse area in Fig.~\ref{subfig:missing_data_comparison_zoom-in}) should have negative filled-in values and therefore should be dark in the image. Neither BALM nor ALM-RPCA's reconstructed images comply to this physical law.

To see this more clearly, we invert the pixel values of the reconstructed image in Fig.~\ref{subfig:missing_data_comparison_negative}.
This is equivalent to inverting the direction of lighting. From the tag of the image, we know that the original lighting is  $-20^{\circ}$ from the subject's right posterior and $40^{\circ}$ from the top, so the inverted light should illuminate the left half of his face from $20^{\circ}$ left frontal and $40^{\circ}$ from below. As is shown in the comparison, only PARSuMi's result revealed what should be correctly seen with a light shining from this direction.


In addition, we reconstruct the 3D depth map with the classic method by \citet{horn1990sfs}. In Fig.~\ref{subfig:shape_comparison}, the shape from PARSuMi reveals much richer depth information than those from the other two algorithms, whose reconstructions appear flattened. This is a known low-frequency bias problem for photometric stereo and it is often caused by errors in the surface normal estimation \citet{nehab2005efficiently}. The fact that BALM and ALM-RPCA produces a flatter reconstruction is a strong indication that their estimations of the surface normal are noisier than that of PARSuMi.


\begin{figure}[htb]
  \centering
  \subfigure[\scriptsize{Comparison of the recovered image}]{
  \includegraphics[width=\linewidth]{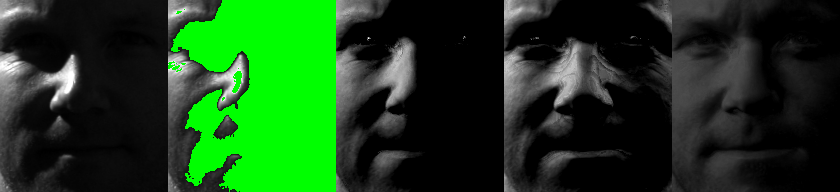}
  \label{subfig:missing_data_comparison}
  }
  \subfigure[\scriptsize{Comparison of the recovered image (details)}]{
  \includegraphics[width=\linewidth]{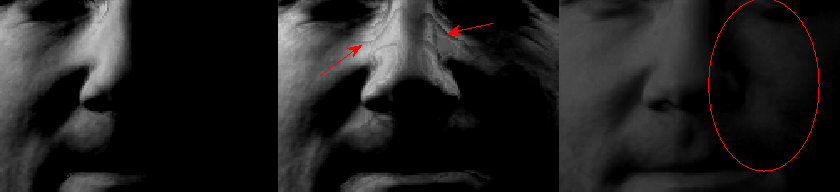}
  \label{subfig:missing_data_comparison_zoom-in}
  }
    \subfigure[\scriptsize{Taking the negative of \subref{subfig:missing_data_comparison_zoom-in} to see the filled-in missing pixels. This is as if the lighting direction is inverted.}]{
  \includegraphics[width=\linewidth]{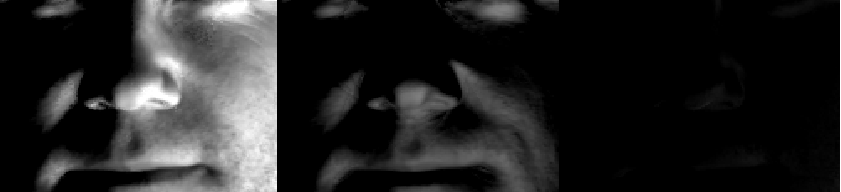}
  \label{subfig:missing_data_comparison_negative}
  }
  \subfigure[\scriptsize{Comparison of the reconstructed 3D surfaces (albedo rendered).}]{
  \includegraphics[width=0.32\linewidth]{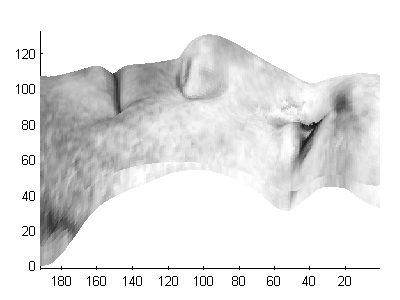}
  \includegraphics[width=0.32\linewidth]{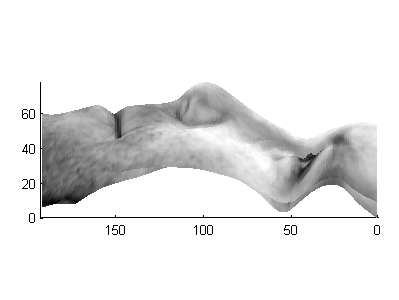}
  \includegraphics[width=0.32\linewidth]{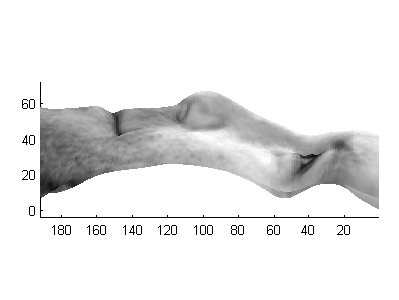}
  \label{subfig:shape_comparison}
  }
  \caption{Qualitative comparison of algorithms on Subject 3. From left to right, the results are respectively for PARSuMi, BALM and ALM-RPCA. In \subref{subfig:missing_data_comparison}, they are preceded by the original image and the image depicting the missing data in green.}\label{fig:YaleB_comparison}
\end{figure}

From our experiments, we find that PARSuMi is able to successfully reconstruct the 3D face for all 38 subjects with little artifacts. As illustrated in Fig.~\ref{fig:YaleB_reconstruction}, our 3D reconstructions of the features seem to reveal the characteristic features of subjects across different ethnic groups.
Moreover, due to the robust $\ell_0$ penalty, PARSuMi is able to effectively recover the input images from many different types of irregularities, e.g. specular regions, different facial expressions, or even image corruptions caused by hardware malfunctions (see Fig.~\ref{fig:illustration_of_PARSuMi_recovery}~and~\ref{fig:comparison_yale_recovery}). This makes it possible for PARSuMi to be integrated reliably into engineering systems that function with minimal human interactions  \footnote{For the best of our knowledge, all previous works that use this dataset for photometric 3D reconstruction manually removed a number of images of poor qualities, e.g.\citet{DelBue2012balm}}.

\begin{figure}[htb]
  \centering
  \subfigure[\scriptsize{Subject 02}]{
  \includegraphics[width=0.2\linewidth]{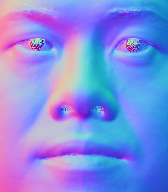}
  \includegraphics[width=0.25\linewidth]{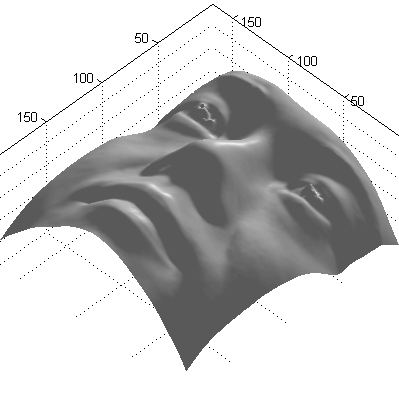}
  \label{subfig:sub02}
  }
  \subfigure[\scriptsize{Subject 5}]{
  \includegraphics[width=0.2\linewidth]{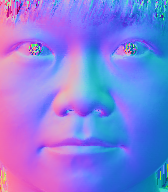}
  \includegraphics[width=0.25\linewidth]{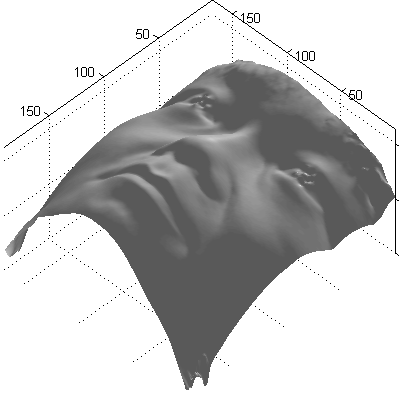}
  \label{subfig:sub05}
  }
  \subfigure[\scriptsize{Subject 10}]{
  \includegraphics[width=0.2\linewidth]{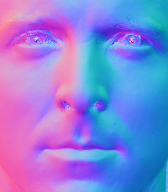}
  \includegraphics[width=0.25\linewidth]{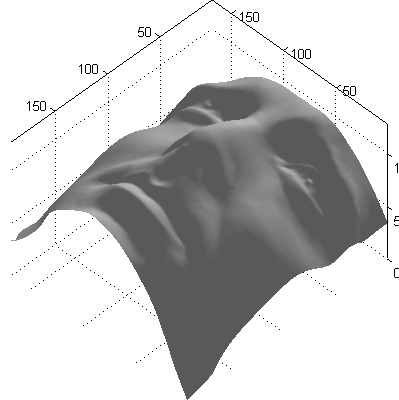}
  \label{subfig:sub10}
  }
  \subfigure[\scriptsize{Subject 15}]{
  \includegraphics[width=0.2\linewidth]{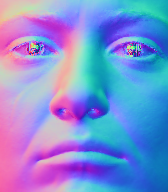}
  \includegraphics[width=0.25\linewidth]{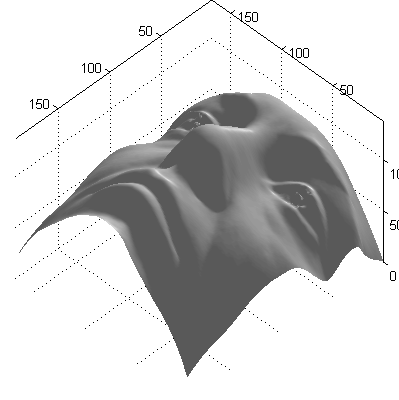}
  \label{subfig:sub15}
  }
  \subfigure[\scriptsize{Subject 12}]{
  \includegraphics[width=0.2\linewidth]{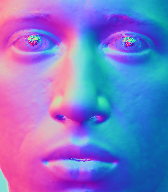}
  \includegraphics[width=0.25\linewidth]{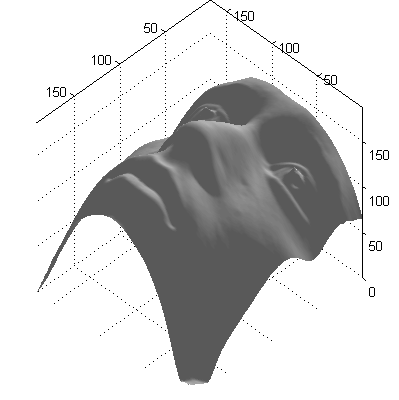}
  \label{subfig:sub12}
  }
  \subfigure[\scriptsize{Subject 22}]{
  \includegraphics[width=0.2\linewidth]{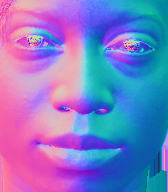}
  \includegraphics[width=0.25\linewidth]{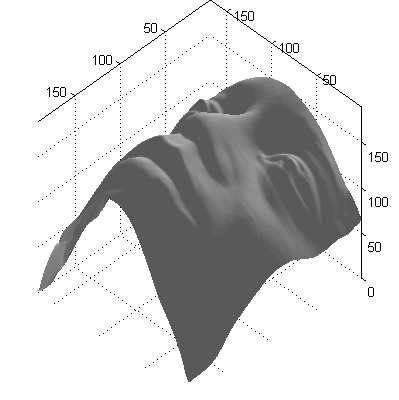}
  \label{subfig:sub22}
  }
  \caption{The reconstructed surface normal and 3D shapes for Asian (first row), Caucasian (second row) and African (third row), male (first column) and female (second column), in Extended YaleB face database.(Zoom-in to look at details)}
  \label{fig:YaleB_reconstruction}
\end{figure}

\begin{figure}[htb]
  \centering
  \subfigure[\scriptsize{Cast shadow and attached shadow are recovered. Region of cast shadow is now visible, and attached shadow is also filled with meaningful negative values.}]{
  \includegraphics[width=0.98\linewidth]{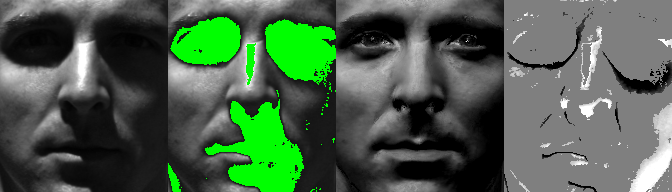}
  \label{subfig:cast_shadow}
  }
  \subfigure[\scriptsize{Facial expressions are set to normal.}]{
  \includegraphics[width=0.98\linewidth]{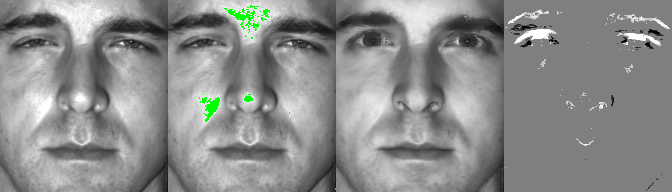}
  \label{subfig:facial expression}
  }\\
  \subfigure[\scriptsize{Rare corruptions in image acquisition are recovered.}]{
  \includegraphics[width=0.98\linewidth]{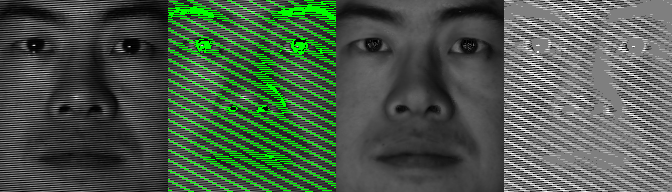}
  \label{subfig:corruption}
  }
  \subfigure[\scriptsize{Light comes 20 degrees from behind and 65 degrees from above).}]{
  \includegraphics[width=0.98\linewidth]{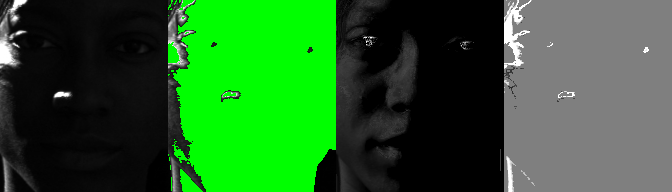}
  \label{subfig:lowsampling}
  }\\
  \caption{Illustrations of how PARSuMi recovers missing data and corruptions. From left to right: original image, input image with missing data labeled in green, reconstructed image and detected sparse corruptions.}
  \label{fig:illustration_of_PARSuMi_recovery}
\end{figure}

\subsection{Speed}
The computational complexity of PARSuMi is cheap for some problems but not for others. Since PARSuMi uses LM\_GN for its matrix completion step, the numerical cost is dominated by either solving the linear system $(J^TJ+\lambda I)\delta=J\mathbf{r}$ which requires
the Cholesky factorization of a potentially dense $mr\times mr$ matrix, or the computation of $J$ which requires solving a small linear system of normal equation involving
the $m\times r$ matrix $N$ for $n$ times. As the overall complexity of $O(\max(m^3r^3, mnr^2))$ scales merely linearly with number of columns $n$ but cubic with $m$ and $r$, PARSuMi is computationally attractive when solving problems with small $m$ and $r$, and potentially large $n$, e.g., photometric stereo and SfM (since the number of images is usually much smaller than the number of pixels and feature points). However, for a typical million by million data matrix as in social networks and collaborative filtering, PARSuMi will take an unrealistic amount of time to run.

Experimentally, we compare the runtime between our algorithm and Wiberg $\ell_1$ method in our Dinosaur experiment in Section~\ref{sec:expt_outlier}. Our Matlab implementation is run on a 64-bit Windows machine with a 1.6 GHz Core i7 processor and 4 GB of memory. We see from Table~\ref{tab:SummaryDinosaur} that there is a big gap between the speed performance. The near 2-hour runtime for Wiberg~$\ell_1$ is discouragingly slow, whereas ours is vastly more efficient. On the other hand, as an online algorithm, GRASTA is inherently fast. Examples in \citet{he2011grasta} show that it works in real time for live video surveillance. However, our experiment suggests that it is probably not appropriate for applications such as SfM, which requires a higher numerical accuracy.

The runtime comparison for the photometric stereo problems is shown in Table~\ref{tab:photometrc}. We remark that PARSuMi is roughly ten times slower than other methods. The pattern is consistent for the YaleB face data too, where PARSuMi takes 23.4 minutes to converge while BALM and RPCA takes only 4.8 and 1.7 minutes respectively.

We note that PARSuMi is currently not optimized for computation. Speeding up the algorithm for application on large scale dataset would  require further effort (such as parallelization) and could be a new topic of research. For instance, the computation of Jacobians $J_i$ and the evaluation of the objective function can be easily done in parallel and the Gauss-Newton update (a positive definite linear system of equations) can be solved using the conjugate gradient method; hence, we do not even need to store the matrix in memory. Furthermore, since PARSuMi seeks to find the best subspace, perhaps using only a small portion of the data columns is sufficient. If the subspace is correct, the rest of the columns can be recovered in linear time with our iterative reweighted Huber regression technique (see Section~\ref{sec:heuristics}). A good direction for future research is perhaps on how to choose the best subset of data to feed into PARSuMi.

\section{Conclusion}
In this paper, we have presented a practical algorithm (PARSuMi) for low-rank matrix completion in the presence of dense noise and sparse corruptions. Despite the non-convex and non-smooth optimization formulation, we are able to derive a set of update rules under the proximal
alternating scheme such that the convergence to a critical point can be guaranteed.  The method was tested on both synthetic and real life data with challenging sampling and corruption patterns. The various experiments we have conducted show that our method is able to detect and remove gross corruptions, suppress noise and hence provide a faithful reconstruction of the missing entries. By virtue of the explicit constraints on both the matrix rank and cardinality,
and the novel reformulation, design and implementation of appropriate algorithms
for the non-convex and non-smooth model,
our method works significantly better than the state-of-the-art algorithms in nuclear norm minimization, $\ell_2$ matrix factorization and $\ell_1$ robust matrix factorization in real life problems such as SfM and photometric stereo.


Moreover, we have provided a comprehensive review of the existing results pertaining to the ``practical matrix completion'' problem  that we considered in this paper. The review covered the theory of matrix completion and corruption recovery, and the theory and algorithms for matrix factorization. In particular, we conducted extensive numerical experiments which reveals (a) the advantages of matrix factorization over nuclear norm minimization when the underlying rank is known, and (b) the two key factors that affect the chance of $\ell_2$-based factorization methods reaching global optimal solutions, namely ``subspace parameterization'' and ``Gauss-Newton'' update. These findings provided critical insights into this difficult problem, upon the basis which we developed PARSuMi as well as its convex initialization.


The strong empirical performance of our algorithm calls for further analysis. For instance, obtaining the theoretical conditions for the convex initialization to yield good support of the corruptions should be plausible (following the line of research discussed in Section~\ref{sec:MC_theory}), and this in turn guarantees a good starting point for the algorithm proper. Characterizing how well the following non-convex algorithm works given such initialization and how many samples are required to guarantee high-confidence recovery of the matrix remain open questions for future study.

Other interesting topics include finding a cheaper but equally effective alternative
to the LM\_GN solver for solving \eqref{eq:Subspace_MC_new},
parallel/distributed computation, incorporating additional structural constraints, selecting optimal subset of data for subspace learning and so on. Step by step, we hope this will eventually lead to a practically working robust matrix completion algorithm that can be confidently embedded in real-life applications.
\vspace{-3mm}
\appendix
\section{Appendix}
\addappheadtotoc
\vspace{-3mm}

\subsection{Proofs}
\begin{proof}[Proof of Proposition~\ref{prop:E_update}] Given a subset $I$ of $\{1,\dots,|\Omega|\}$ with cardinality at most $N_0$
such that $b_I\not = 0$. Let $J = \{1,\dots,|\Omega|\} \backslash I$.
Consider the problem \eqref{eq:MC_E2} for $x\in \R^{|\Omega|}$ supported on $I$, we
get the following:
\begin{eqnarray*}
 v_I := \min_{x_I} \Big\{  \norm{x_I- b_I}^2  + \norm{b_J}^2 \mid   \norm{x_I}^2 -K_E^2\leq 0\Big\},
\end{eqnarray*}
which is a convex minimization problem whose optimality conditions are given by
\begin{eqnarray*}
  x_I - b_I + \mu \, x_I = 0, \; \mu (\norm{x_I}^2-K_E^2) =0, \; \mu \geq 0
\end{eqnarray*}
where $\mu$ is the Lagrange multiplier for the inequality constraint.
First consider the case where $\mu > 0$. Then we get $x_I=K_E b_I/\norm{b_I}$, and
$1+\mu = \norm{b_I}/K_E$ (hence $\norm{b_I} > K_E$). This implies that
$
 v_I =   \norm{b}^2 +K_E^2 - 2 \norm{b_I}K_E.
$
On the other hand, if $\mu=0$, then we have $x_I = b_I$ and
$v_I = \norm{b_J}^2 = \norm{b}^2-\norm{b_I}^2$.
Hence
\begin{eqnarray*}
 v_I = \left\{ \begin{array}{ll}
   \norm{b}^2 +K_E^2 - 2 \norm{b_I}K_E & \mbox{if $\norm{b_I} > K_E$}
\\[5pt]
 \norm{b}^2-\norm{b_I}^2 & \mbox{if $\norm{b_I}\leq K_E$.}
\end{array} \right.
\end{eqnarray*}
In both cases,
it is clear that $v_I$ is minimized if $\norm{b_I}$ is maximized.
Obviously $\norm{b_I}$ is maximized if $I$ is chosen to be
the set of indices corresponding to the $N_0$ largest components of $b$.
\qed
\end{proof}

\begin{proof}[Proof of Theorem~\ref{thm:critical_pt}]
(a) The equality in \eqref{eq-thm1-1} follows directly
from \eqref{eq-LhL}.
By the minimal property of $\hx_{k+1}$, we have that
\begin{eqnarray}
 \widehat{L}(\hx_{k+1};x_k,y_k) \;\leq\;  \widehat{L}(\xi;x_k,y_k) \quad \forall\;
\xi\in \cX.
\label{eq-thm1-a2}
\end{eqnarray}
Thus when $\xi = x_k$, we get
$\widehat{L}(\hx_{k+1};x_k,y_k)\leq \widehat{L}(x_k;x_k,y_k) = L(x_k,y_k)$,
and the required inequality in \eqref{eq-thm1-1} follows.
On the other hand, the inequality \eqref{eq-thm1-2} follows readily from the
definition of $x_{k+1}$.
\\[5pt]
(b) If $x_{k+1} = \tx_{k+1}$, then
from the definition of $x_{k+1}$ and  \eqref{eq-thm1-1}, we have that,
\begin{equation}\label{eq-thm1-b1}
\begin{aligned}
   &L(x_{k+1},y_k) +  \frac{1}{2}\norm{x_{k+1}-x_k}_S^2\\
\leq& L(\hx_{k+1},y_k) +\frac{1}{2}\norm{\hx_{k+1}-x_k}_S
  \leq  L(x_k,y_k).
\end{aligned}
\end{equation}
On the other hand, if $x_{k+1} =\hx_{k+1}$, we have from
\eqref{eq-thm1-1}  that
\begin{eqnarray}
   L(x_{k+1},y_k) +  \frac{1}{2}\norm{x_{k+1}-x_k}_S^2
 & \leq & L(x_k,y_k).
\label{eq-thm1-b2}
\end{eqnarray}
By the minimal property of $y_{k+1}$, we have that $ \forall\;\eta\in\cY$
\begin{eqnarray*}
   L(x_{k+1},y_{k+1}) +  \frac{1}{2}\norm{y_{k+1}-y_k}_T^2
 & \leq & L(x_{k+1},\eta) +  \frac{1}{2}\norm{\eta-y_k}_T^2.
\end{eqnarray*}
In particular, when $\eta=y_k$, we get
\begin{eqnarray}
   L(x_{k+1},y_{k+1}) +  \frac{1}{2}\norm{y_{k+1}-y_k}_T^2 \;\leq\;
 L(x_{k+1},y_k).
\label{eq-thm1-c}
\end{eqnarray}
By combining \eqref{eq-thm1-b1}-\eqref{eq-thm1-b2} and \eqref{eq-thm1-c}, we get
the inequality \eqref{eq-thm1-3}.
\\[5pt]
(c) Note that by using the result in part (b), we also have
$\lim_{k'\rightarrow\infty} x_{k'+1} = \bar{x}$ and $\lim_{k'\rightarrow\infty} y_{k'+1} = \bar{y}$.
From \eqref{eq-thm1-1}, \eqref{eq-thm1-2} and \eqref{eq-thm1-a2}, we have $\forall\; k\geq 0,\; \xi\in\cX$
\begin{eqnarray}
 L(x_{k+1},y_k) + \frac{1}{2}\norm{x_{k+1}-x_k}_S^2 \leq \widehat{L}(\xi;x_k,y_k).
\end{eqnarray}
Thus $\forall\; \xi\in\cX$
\begin{eqnarray}
 \limsup_{k'\rightarrow\infty} f(x_{k'+1}) + q(\bar{x},\bar{y}) \leq
f(\xi) + Q(\xi;\bar{x},\bar{y}).
\end{eqnarray}
By taking $\xi = \bar{x}$, we get
\begin{equation}
 \limsup_{k'\rightarrow\infty} f(x_{k'+1})  \leq
f(\bar{x}) + Q(\bar{x};\bar{x},\bar{y}) - q(\bar{x},\bar{y})
= f(\bar{x}) .
\end{equation}
On the other hand, since $f$ is lower semicontinuous, we have
that $\liminf_{k'\rightarrow\infty} f(x_{k'+1})  \geq f(\bar{x})$.
Thus $\lim_{k'\rightarrow\infty}$ $f(x_{k'+1})  = f(\bar{x})$.
Similarly, we can show that $\lim_{k'\rightarrow\infty}$ $g(y_{k'+1})  = g(\bar{y})$.
As a result, we have
\begin{eqnarray}
 \lim_{k'\rightarrow\infty} L(x_{k'+1},y_{k'+1}) = L(\bar{x},\bar{y}).
\end{eqnarray}
Since $\{ L(x_k,y_k)\}$ is a nonincreasing sequence, the above result implies that
$$
  \lim_{k\rightarrow\infty} L(x_k,y_k) = L(\bar{x},\bar{y}) = \inf_{k} L(x_k,y_k).
$$
Also, \eqref{eq-thm1-b1}-\eqref{eq-thm1-b2} and \eqref{eq-thm1-c} implies that
$$
  \lim_{k\rightarrow\infty} L(x_{k+1},y_k) = L(\bar{x},\bar{y}).
$$
From \eqref{eq-thm1-1} and \eqref{eq-thm1-2}, we have
\begin{align*}
L(x_{k+1},y_k) + \frac{1}{2}\norm{x_{k+1}-x_k}_S^2\\
 \;\leq\; \widehat{L}(\hx_{k+1};x_k,y_k)
\;\leq\; L(x_k,y_k).
\end{align*}
Thus $\lim_{k\rightarrow\infty} \widehat{L}(\hx_{k+1};x_k,y_k) = L(\bar{x},\bar{y})$.

Now by \eqref{eq-thm1-1} and \eqref{eq-thm1-2} again, we have
\begin{align*}
 \frac{1}{2}\norm{x_{k+1}-x_k}_S^2 + \frac{1}{2}\norm{\hx_{k+1}-x_k}_{M-A^*A-S}^2\\
\leq \widehat{L}(\hx_{k+1};x_k,y_k)  - L(x_{k+1},y_k).
\end{align*}
Thus
$
\lim_{k\rightarrow\infty} \norm{\hx_{k+1}-x_k}_{M-A^*A-S}^2 =0$. Since
$M - A^*A-S \succ 0$, we also get $\lim_{k\rightarrow\infty} \norm{\hx_{k+1}-x_k}^2 =0$.
\\[5pt]
(d) From the optimality of $\hx_{k+1}$, we have that
\begin{align*}
 0 &\in \partial \widehat{L}(\hx_{k+1}; x_k,y_k) \\
 &=  \partial f(\hx_{k+1}) + A^*(A x_{k}+By_k-c) +
M (\hx_{k+1}-x_k)
\\
&=  \partial f(\hx_{k+1}) + A^*(A\hx_{k+1}+By_{k+1}-c)
 - \Delta x_{k+1}
\end{align*}
where $\Delta x_{k+1} = -(M-A^*A) (\hx_{k+1}-x_k)  - A^*B(y_k-y_{k+1})$.
Thus
\begin{eqnarray}
\Delta x_{k+1} \in \partial_x L(\hx_{k+1},y_{k+1}).
\end{eqnarray}
From the optimality of $y_{k+1}$, we have that
\begin{eqnarray*}
 0 &\in& \partial g(y_{k+1})
+ B^*(A x_{k+1}+By_{k+1}-c) + T (y_{k+1}-y_k)
\\[5pt]
&=& \partial_y L (\hx_{k+1},y_{k+1}) + T (y_{k+1}-y_k)
+ B^*A (x_{k+1}-\hx_{k+1})
\end{eqnarray*}
Hence
$
\Delta y_{k+1} := -  T (y_{k+1}-y_k)-B^*A (x_{k+1}-\hx_{k+1})  \in \partial_y L(\hx_{k+1},y_{k+1}).
$

From part (b) and (c), we have
that
\begin{align*}
&\lim_{k'\rightarrow \infty} \norm{\hx_{k'+1}-x_{k'}} = 0,&&
\lim_{k'\rightarrow \infty}\norm{\hx_{k'+1}-x_{k'+1}} =0,\\
&\lim_{k'\rightarrow \infty}\hx_{k'+1} = \bar{x},&&
\lim_{k'\rightarrow \infty} y_{k'+1} = \bar{y}.
\end{align*}
Thus
$$
 \lim_{k'\rightarrow \infty} \Delta x_{k'+1} = 0 =  \lim_{k'\rightarrow \infty} \Delta y_{k'+1}.
$$
By the closedness property of $\partial L$ \cite[Proposition 2.1.5]{clarke1990optimization}, we get
$$
 (0,0) \in \partial L(\bar{x},\bar{y}).
$$
Thus $(\bar{x},\bar{y})$ is a stationary point of L. \qed
\end{proof}

\subsection{Software/code used}
The point cloud in Fig.~\ref{fig:Point cloud} are generated using VincentSfMToolbox \citep{vincentsSFMToolbox}. Source codes of BALM, GROUSE, GRASTA, Damped Newton, Wiberg, LM\_X used in the experiments are released by the corresponding author(s) of \citet{DelBue2012balm,balzano2010grouse,he2011grasta,Damped_Newton_2005,WibergL2,Wu_photometric} and \citet{Subspace_ChenPei_2008}. In particular, we are thankful that \citet{balzano2010grouse} and \citet{Wu_photometric} shared with us a customized version of GROUSE and ALM-RPCA that are not yet released online. For Wiberg $\ell_1$ \citep{WibergL1}, we have optimized the computation for Jacobian and adopted the commercial LP solver: cplex. The optimized code performs identically to the released code in small scale problems, but it is beyond the scope for us to verify for larger scale problems. In addition, we implemented SimonFunk's SVD ourselves. The ALS implementation is given in the released code package of LM\_X. For OptManifold, TFOCS and CVX, we use the generic optimization packages released by the author(s) of \citet{yin2013orthogonality,becker2011tfocs,gb08cvx} and customize for the particular problem. For NLCG, we implement the derivations in \citet{srebro2003weighted} and used the generic NLCG package \citep{nlcg}.

\subsection{Additional experimental results}

Illustration of the decomposition on Subject 3 of Extended YaleB dataset is given in Fig.~\ref{fig:YaleB_decomposition}. Additional qualitative comparisons on the recovery of the image is given in Fig.~\ref{fig:comparison_yale_recovery}.
\begin{figure*}[htb]
  \centering
  \subfigure[\scriptsize{The 64 original face images}]{
  \includegraphics[width=0.48\linewidth]{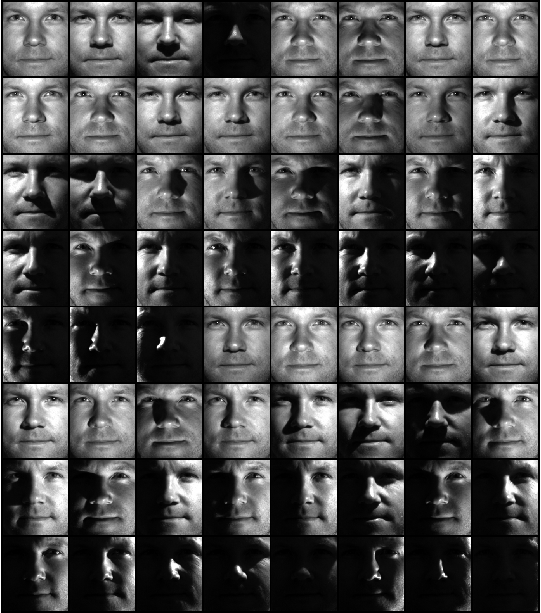}
  \label{subfig:original_images}
  }
  \subfigure[\scriptsize{Input images with missing data (in green)}]{
  \includegraphics[width=0.48\linewidth]{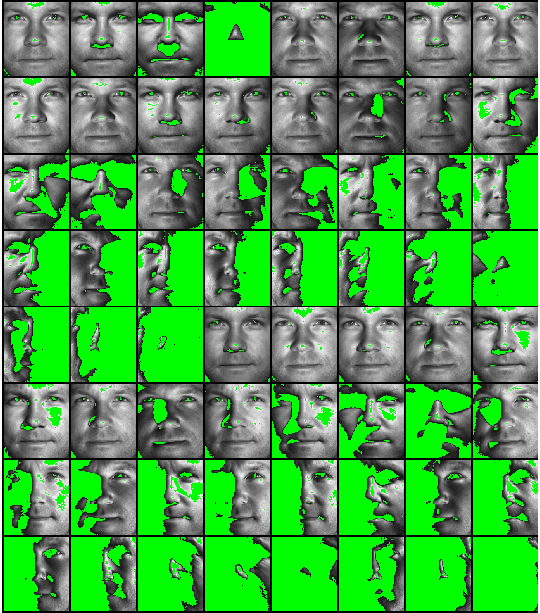}
  \label{subfig:input_missing}
  }
  \subfigure[\scriptsize{The 64 recovered rank-3 face images}]{
  \includegraphics[width=0.48\linewidth]{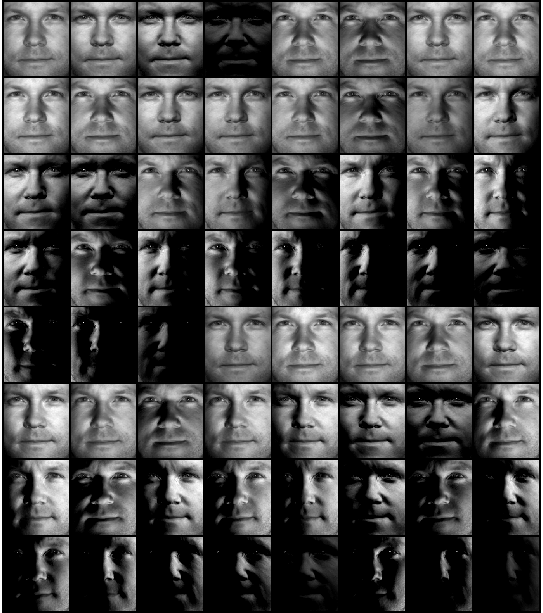}
  \label{subfig:recovered_faces}
  }
  \subfigure[\scriptsize{Sparse corruptions detected}]{
  \includegraphics[width=0.48\linewidth]{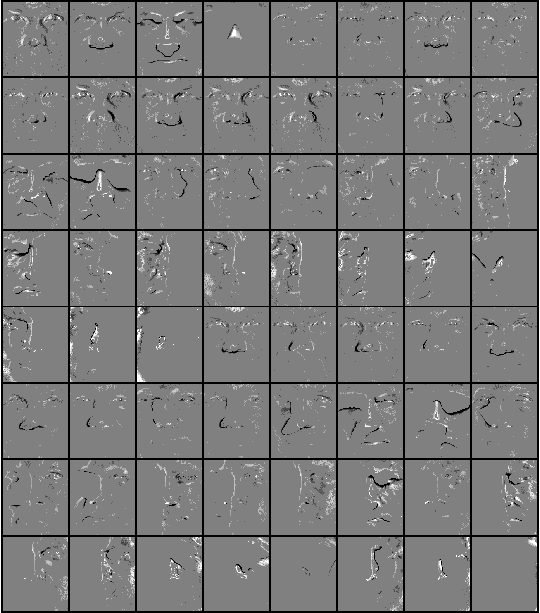}
  \label{subfig:sparse_corruptions}
  }\\
  \caption{\small{Results of PARSuMi on Subject 3 of Extended YaleB. Note that the facial expressions are slightly different and some images have more than 90\% of missing data. Also note that the sparse corruptions detected unified the irregular facial expressions and recovered those highlight and shadow that could not be labeled as missing data by plain thresholding.}}
  \label{fig:YaleB_decomposition}
\end{figure*}
\begin{figure}
  \centering
  \includegraphics[width=\linewidth]{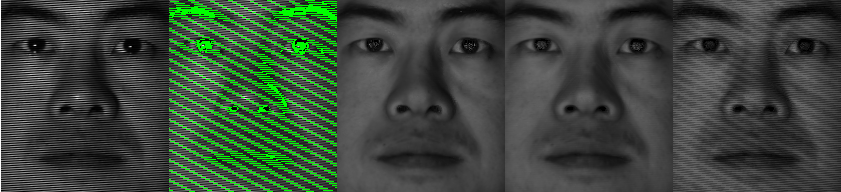}\\
  \includegraphics[width=\linewidth]{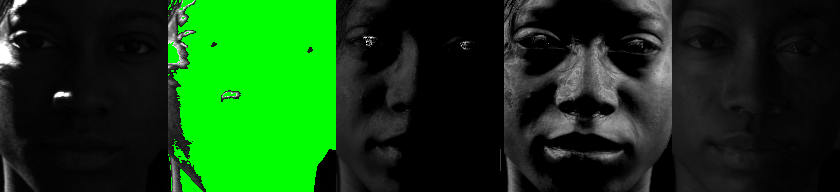}
  \caption{Additional comparisons in the quality of face image recovery. From left to right, they are original image, missing data mask (in green), results for PARSuMi, BALM and missing-RPCA.}\label{fig:comparison_yale_recovery}
\end{figure}

\subsection{The lower bounds in the experiments}
\begin{itemize}
  \item The lower bound in Fig.~\ref{fig:mf_exp1}: the lower bound is obtained by the data set that contains less than $r$ data points per-column and per-row. It is clear from \citet{kiraly2012algebraic} that this is an easy-to-check necessary condition of recoverability.
  \item \emph{The oracle RMSE for Phase Diagram}: We also adapt the oracle lower bound from \citet{CandesNoise} to represent the theoretical limit of recovery accuracy under noise. Our extended oracle bound under both sparse corruptions and Gaussian noise is:
        \begin{equation}\label{eq:OracleBound}
            \mathrm{RMSE}_{oracle}=\sigma\sqrt{\frac{(m+n-r)r}{p-e}},
        \end{equation}
        This is used for benchmarking in our phase diagram experiments.
  \item \emph{The oracle angular error in Table~\ref{tab:photometrc}:} For the Caesar and Elephant experiments, we use \eqref{eq:OracleBound} (ignoring corruptions by taking $e=0$) but transformed it by taking
      $$\arcsin \sqrt{1-(n\cdot \hat{n})^2},$$
      where $\hat{n}$ is the surface normal obtained by an oracle projection of the noisily observed image.
\end{itemize}

\subsection{Summary of parameters used in the experiments}\label{sec:parameters}

\begin{itemize}
  \item \emph{Parameters in our formulation}:
We assume $r$ (the underlying rank) to be known. $N_0$ is chosen to be an upper bound of the number of corrupted entries.
In experiments, we use 120\% of the actual number of corruptions. In practice, we should choose $N_0=0.1|\Omega|$ or $0.15|\Omega|$.
$\epsilon=1e-10$ (almost negligible).$K_E=20\sqrt{N_0\times \mathrm{median}(\cP)_{\Omega}(\widehat{W})}$ (very large, negligible).
In theory, we only need $\epsilon>0$ and $K_E<\infty$ to ensure the convergence. In practice, unless it is meaningful to choose an effective $K_E$, we will choose it large enough so that it has no impact on the optimization.
  \item \emph{Parameters for PARSuMi}:
$\beta_1=\beta_2=\frac{1e-3}{\sqrt{\max{m,n}}}$. For Algorithm~\ref{alg:LM_GN}, $\rho=10$ and initial $\lambda=1e-6$.
  \item \emph{Parameters for APG}:
$\gamma = \frac{1}{\sqrt{\max{m,n}}}$
$\lambda = 0.2$
\end{itemize}





{
\bibliographystyle{spbasic}
\bibliography{ARSuMi}   
}

\end{document}